%% file: main.tex
\documentclass[sigplan,10pt]{acmart} %nonacm, 
\usepackage{tikz}
\usepackage{amsmath}

\usepackage{filecontents}

\usepackage{epsfig}
\usepackage{subcaption}
\usepackage{color}
\usepackage{amsmath, amsfonts}
\usepackage{multirow}
\usepackage[ruled,linesnumbered,noresetcount,noend]{algorithm2e}
\SetAlFnt{\footnotesize}
\usepackage{xurl}
\usepackage{booktabs}

\usepackage{makecell}

\usepackage{enumitem}

\usepackage{comment}
\usepackage{listings}
\usepackage{cleveref}
\usepackage{caption}
\usepackage[belowskip=0pt,aboveskip=0pt]{caption}
\usepackage{amsthm}
\newtheorem{definition}{Definition}

\newtheorem{theorem}{Theorem}

\usepackage{makecell}

\newcommand{\comments}[2]{{\bf[\textcolor{red}{#1}: \textcolor{blue}{#2}]}}

\newcommand{\stile}{{STile}}
\newcommand{\sload}{{SRead}}
\newcommand{\swrite}{{SWrite}}
\newcommand{\NX}[1]{{\comments{Ningxin}{#1}}}

\newcommand{\revise}[1]{\textcolor{black}{#1}}
\newcommand{\eg}{{e.g.,\ }}

\newcommand{\ie}{{i.e.,\ }}

\newcommand{\para}[1]{{\bf \noindent #1 \hspace{2pt}}}
\usepackage{nopageno}

\usepackage[normalem]{ulem}
\newcommand{\sysname}{{PIT}} %{{Spider}} %SparDA
 
\newcommand{\adjp}{permutation invariant}

\usepackage{listings}
\usepackage{xcolor}

\definecolor{codegreen}{rgb}{0,0.6,0}
\definecolor{codegray}{rgb}{0.5,0.5,0.5}
\definecolor{codepurple}{rgb}{0.58,0,0.82}
\definecolor{backcolour}{rgb}{0.95,0.95,0.92}

\lstdefinestyle{mystyle}{
    backgroundcolor=\color{white},   
    commentstyle=\color{codegreen},
    keywordstyle=\color{magenta},
    numberstyle=\tiny\color{codegray},
    stringstyle=\color{codepurple},
    basicstyle=\ttfamily\footnotesize,
    breakatwhitespace=false,         
    breaklines=true,                 
    captionpos=b,                    
    keepspaces=true,                 
    numbers=left,                    
    numbersep=5pt,                  
    showspaces=false,                
    showstringspaces=false,
    showtabs=false,                  
    tabsize=2
}

\begin{document}

%\title{\sysname{}: Accelerating Dynamic Sparse Deep Neural Networks via Sparse-Dense Transformation}

\date{}
\title[]{\sysname{}: Optimization of Dynamic Sparse Deep Learning Models via Permutation Invariant Transformation}
\author{Ningxin Zheng$^*$, Huiqiang Jiang$^*$, Quanlu Zhang, Zhenhua Han, Lingxiao Ma, Yuqing Yang}\author{Fan Yang, Chengruidong Zhang, Lili Qiu, Mao Yang, Lidong Zhou}
\thanks{$^*$ Equal contribution.}
\affiliation{\vspace{0.1cm}\Large\institution{Microsoft Research}\country{}}
\renewcommand{\shortauthors}{}
% \author[1*]{Ningxin Zheng}
% \author[1*]{Huiqiang Jiang}
% \author[1]{Quanlu Zhang}
% \author[1]{Zhenhua Han}
% \author[1]{Lingxiao Ma}
% \author[1]{Yuqing Yang}
% \author[1]{Fan Yang}
% \author[1]{Chengruidong Zhang}
% \author[1]{Lili Qiu}
% \author[1]{Mao Yang}
% \author[1]{Lidong Zhou}
% \affil[1]{Microsoft Research}

%\title{\huge \sysname{}: Efficient Deep Learning Computation with Dynamic Sparsity via Permutation Invariant Transformation}

\begin{abstract}
Dynamic sparsity, where the sparsity patterns are unknown until runtime, poses a significant challenge to deep learning. 
The state-of-the-art sparsity-aware deep learning solutions are restricted to pre-defined, static sparsity patterns due to significant overheads associated with preprocessing. Efficient execution of dynamic sparse computation often faces the misalignment between the GPU-friendly tile configuration for efficient execution and the sparsity-aware tile shape that minimizes coverage wastes (non-zero values in tensor). 
%between the inefficient GPU execution on sparsity granularity and the high wastes of GPU-efficient tile shapes due to the coarse-grained coverage.

In this paper, we propose \sysname{}, a deep-learning compiler for dynamic sparsity. \sysname{} proposes a novel tiling mechanism that leverages Permutation Invariant Transformation (PIT), a mathematically proven property, to transform multiple sparsely located micro-tiles into a GPU-efficient dense tile without changing the computation results, thus achieving both high GPU utilization and low coverage waste. %The correctness of such transformation is mathematically guaranteed by Permutation Invariant Transformation (PIT). %Enabling PIT for dynamic sparsity requires finding feasible PIT, generating efficient GPU kernels, and fast execution of PIT at runtime. 
%We solve the key challenges to enable PIT for dynamic sparsity. 
Given a model, \sysname{} first finds feasible PIT rules for all its operators and generates efficient GPU kernels accordingly. At runtime, with the \revise{novel} \sload{} and \swrite{} primitives, PIT rules can be executed extremely fast % seems duplicated with the sentence below
to support dynamic sparsity in an online manner. 
% If sparse tensors are given in a compressed format, %dense format, %what does dense format mean? it is a dense tensor or compressed sparse tensor? 
% \sysname{} can even detect the sparse coordinates of micro-tiles on the fly with negligible overheads. % it would be great if we can quantify the overheads.
Extensive evaluation on diverse models shows that \sysname{} can accelerate dynamic sparsity computation by up to 5.9x \revise{(average 2.43x)} over state-of-the-art compilers. % with negligible transformation overheads. 

\end{abstract}

% % \settopmatter{printfolios=true}
% \begin{CCSXML}
% <ccs2012>
%    <concept>
%        <concept_id>10010520.10010521.10010542.10010294</concept_id>
%        <concept_desc>Computer systems organization~Neural networks</concept_desc>
%        <concept_significance>500</concept_significance>
%        </concept>
%  </ccs2012>
% \end{CCSXML}

% \ccsdesc[500]{Computer systems organization~Neural networks}



\acmYear{2023}\copyrightyear{2023}
\setcopyright{acmlicensed}
\acmConference[SOSP '23]{ACM SIGOPS 29th Symposium on Operating Systems Principles}{October 23--26, 2023}{Koblenz, Germany}
\acmBooktitle{ACM SIGOPS 29th Symposium on Operating Systems Principles (SOSP '23), October 23--26, 2023, Koblenz, Germany}
\acmPrice{15.00}
\acmDOI{10.1145/3600006.3613139}
\acmISBN{979-8-4007-0229-7/23/10}

\maketitle

\input{Introduction_2}
\input{motivation_2}

\input{Design}
\input{Implementation}
\input{Evaluation}
\input{Relatedwork}

\input{Discussion}

\vspace{-0.2cm}
\section{Conclusion}
\sysname{} takes a principled approach to support efficient execution of dynamically sparse models on commodity accelerators, based on \emph{permutation invariant transformation}, a property commonly existing in deep learning computations. With this property, \sysname{} constructs computation-efficient dense tiles from hardware-friendly micro-tiles in an online manner. \sysname{} demonstrates a novel and effective way of handling dynamic sparsity, a growing trend in deep learning. The extensive evaluation shows \sysname{} can accelerate dynamic sparsity computation in both inference and training by up to 5.9x over state-of-the-art solutions.
%\property{} is deduced by applying permutation invariant on sparse computation. Based on \property{}, we propose \stile{} a core abstraction that naturally connects dynamic sparsity with accelerator-aligned dense computations, leading to a new compiling framework for dynamic sparsity. We envision that \sysname{} can greatly accelerate the innovations of more biological brain-like neural networks
% dynamically sparse models
%in deep learning community.

\section*{\revise{Acknowledgement} }
\revise{We thank our shepherd, Thomas Pasquier, for his valuable suggestion that improves the quality of our paper. We also thank the anonymous reviewers for their constructive feedback. Quanlu Zhang, Zhenhua Han, and Yuqing Yang are the corresponding authors.}

\bibliographystyle{plain}
\bibliography{references}

\end{document}

%% file: Introduction_2.tex
\section{Introduction}
Tensor is the key data abstraction in deep learning. It provides a powerful and flexible way to represent contents including images, audio, and language sentences. %The computation of deep learning models is composed of operators, which are mostly linear algebra computations (e.g., matrix multiplication) over tensors. 
Deep learning computation mostly involves operations over tensors (\eg matrix multiplications). 
To efficiently execute deep learning models, accelerators like GPUs have been designed to perform these tensor operations in parallel. GPUs usually have multiple Stream Multiprocessors (SM), each with hundreds of CUDA cores that can simultaneously perform arithmetic operations on different portions of the tensor. To efficiently utilize these parallel processing capabilities, deep learning compilers use tiling to break up tensors into smaller, regular tensor slices, a.k.a. tiles, that can be processed in parallel by multiple SMs. Tiling has been demonstrated to be a key optimization technique for %dense 
tensor computations by effectively exploiting data locality and parallelism.

%However, progress in deep learning algorithms has revealed a growing number of sparse computations. 
Recent developments suggest that deep learning computations are increasingly \emph{sparse}, \ie 
%Sparse computations involve 
operations on tensors with many zeros. In addition to sparse model weights, which are often static and known a priori, more sparsity patterns are found to depend on inputs and are only known at runtime, \ie dynamic sparsity. For example, large language models (\eg GPT~\cite{brown2020language, OpenAI2023GPT4TR}, OPT~\cite{zhang2022opt}) exhibit various types of dynamic sparsity. Firstly, transformer models (and other types of models) show inherent dynamic sparsity in their inputs, activation and gradients~\cite{Liu2023TenLW}. Only a very small number of elements in the activation map are non-zero~\cite{li2022large} (\eg 3.0\% for T5-Base~\cite{raffel2020exploring} and 6.3\% for ViT-B/16~\cite{dosovitskiy2021an}). Secondly,  dynamic sparsity is being leveraged to help further scale DNN models to a large size. Most existing models with more than one trillion parameters adopt the mixture of experts (MoE) structure, which activates experts sparsely and dynamically depending on the input~\cite{fedus2021switch}. Moreover, dynamic sparse training, which dynamically prunes less important connections in the model during training, is attracting more attention for its superior computational efficiency~\cite{Liu2023TenLW,Evci2019RiggingTL}.

\begin{figure}[t]
    \centering
    \includegraphics[width=0.5\textwidth]{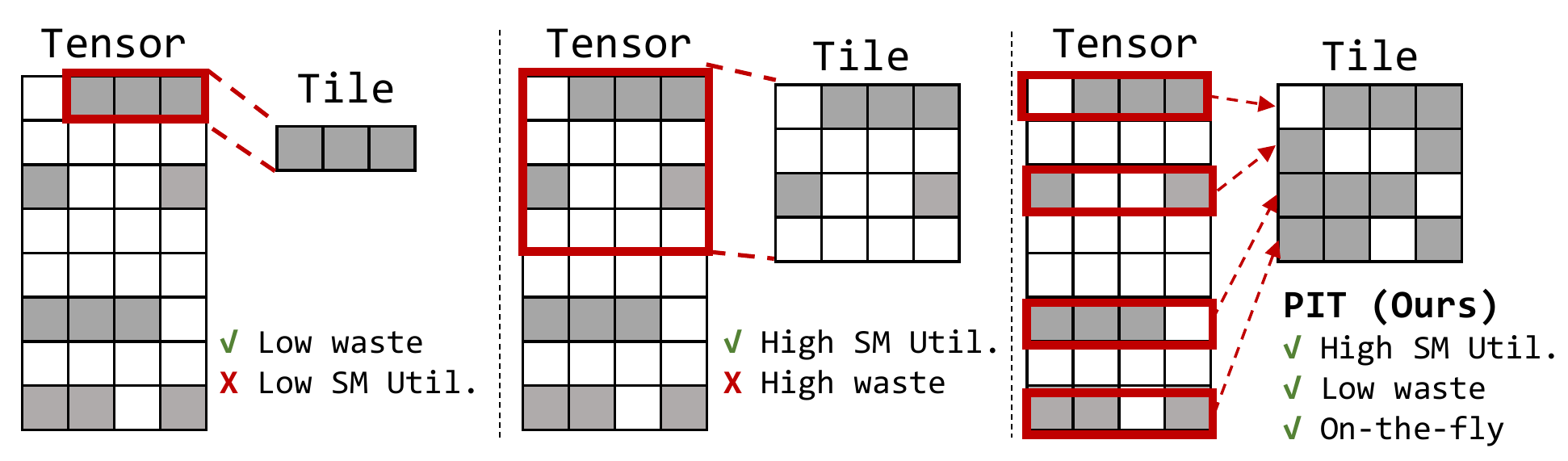}
    \caption{Tiling for dynamic sparsity (shaded blocks are non-zero values). Misalignment between sparsity granularity and the tile shape of efficient GPU kernels.}
    \label{fig:tile}
\end{figure}

It is a daunting challenge to efficiently compute deep learning models in the presence of dynamic sparsity. % on GPUs is a daunting challenge. 
Modern deep learning compilers~\cite{tvm, triton, roller}, including the sparsity aware compilers~\cite{sparta, tvmsparse}, leverage the sparsity patterns in deep learning models known at the compile time to find the right kernel configuration. They are infeasible to be used for dynamic sparsity due to high compilation time. Some fine-grained sparsity solutions can be executed at runtime~\cite{sputnik,cusparse}, but they use special sparse formats (\eg CSR: Compressed Sparse Row) incurring high overheads for format conversion.

Efficient GPU computation requires loading tiled data into shared memory for data reuse. However, the shape of hardware efficient tiles is usually misaligned with diverse and dynamic sparsity patterns. As shown in the left and mid figures in \autoref{fig:tile}, although sparsity-aligned tiles incur low coverage waste, they run slow on GPU due to low SM utilization and data reuse. But, since sparse values are usually non-continuously located, the tile shapes used in efficient GPU kernels cannot avoid covering a lot of zeros, leading to waste in sparse computation. Existing solutions for sparsity try to find better trade-offs between efficient tiling and sparsity shape alignment. They either require knowledge at the compile time, or incur significant overheads at runtime due to format conversation.

%with high conversion costs, thus only working for static sparsity patterns.

% 1. Inconsistent
% 2. Hardware-friendly
% While state-of-the-art compilers for sparse computation can use code specialization (e.g., SparTA) or heavy preprocessing (e.g., ???) to execute fine-grained sparsity computations, these approaches only work for static sparsity due to the high overhead involved. The overhead precludes use at runtime.

In this paper, we present \sysname{}, a compiler for the efficient execution of deep learning models with dynamic sparsity. \sysname{} resolves the misalignment of tile shapes by constructing GPU-efficient tiles with multiple sparsely located ``micro-tiles'' (\eg the right figure in \autoref{fig:tile}). The construction is performed at runtime and its correctness is guaranteed mathematically by ``Permutation-Invariant Transformation'', a property widely existed in DL operators but is not well-exploited for dynamic sparsity. Most DL operators have one or more dimensions (we call them  \sysname{}-axis), whose computation can be arbitrarily reordered, without affecting the result. For example, in matrix multiplication (Matmul) whose tensor expression is $C[m,n]$+=$A[m,k] \times B[k,n]$, the columns of $A$ along with the rows of $B$ (\ie the $k$ dimension) can be permuted in any order without affecting the computation result. The rows of $A$ along with the rows of $C$ (\ie the $m$ dimension) can also be permuted without affecting the computation. By merging micro-tiles into a dense tile along a certain \sysname{}-axis,  \sysname{} achieves the best of both worlds. It can leverage the efficient execution of dense tile at GPU SMs while achieving low wasted computation with the fine-grained micro-tile coverage of non-zeros.

%\sysname{} solves multiple challenges to enable the permutation-invariant transformation mechanism for dynamic sparsity. \sysname{} needs to find all feasible \sysname{}-axes of the operators for a given model.

A key challenge to leverage permutation-invariant transformation is the runtime transformation overhead. %Dynamic sparsity requires ultra-low overhead on the transformation for efficient execution in GPUs. 
\sysname{} solves the challenge by introducing \emph{SRead} and \emph{SWrite}, two data rearrangement primitives for loading and storing between sparsely located data in GPU global memory and the shared memory. By piggybacking the data rearrangement on the data movement across the memory hierarchy during standard tensor computation, \emph{SRead} and \emph{SWrite} can achieve almost no extra overhead for permutation-invariant transformation. %Even when the data is presented in a dense tensor, 
Even when the coordinates of sparse values in the tensors are unknown, \sysname{} can detect them quickly in an online manner. 
Interestingly, by treating it as a special case of dynamic sparsity, \sysname{} can also work effectively on deep learning models with static sparsity patterns where the misalignment of tile shapes is also a source of inefficiency. 

%\sysname{} generally works for both dynamic sparsity and static sparsity computation, because static sparsity can be taken as a special case of dynamic sparsity.

%First, given a DL model, we need to find all the  \sysname{}-axes for its operators and generate a list of permutation-invariant transformation rules. A permutation-invariant transformation rule of an operator includes a  \sysname{}-axis, a micro-tile shape, and an efficient GPU kernel under a specific dense tile size. Second, the micro-tiles need to be efficiently loaded into the shared memory of GPU SMs before computation, and then written back to the global memory after computation, which is achieved by  \sysname{}'s SRead and SWrite primitives. Third, when the input is given in a dense format (e.g., tensor),  \sysname{} needs to quickly extract the micro-tile coordinates at runtime, as the prerequisite of permutation-invariant transformation.
% Exploit memory hierarchy to efficient GPU execution of PIT. 

%We have implemented \sysname{} on PyTorch. 
We extensively evaluate \sysname{} on five representative models (\ie Switch Transformer~\cite{wolf2020transformers}, OPT~\cite{zhang2022opt}, BERT~\cite{bert}, LongFormer~\cite{longformer}, Museformer~\cite{museformer}) and find that it significantly speeds up both inference and training while using less memory.  In terms of inference, \sysname{} achieve up to 5.9x speedup with less memory consumption compared to state-of-art solutions (\S\ref{sec:eval_end_2_end}). Additionally, we also applied \sysname{} to speed up the large language model training and sparse training. \sysname{}  achieves up to 1.8x speedup on the OPT training, and up to 2.4x speedup for the sparse training compared to the previous solutions. Dynamic sparsity is becoming increasingly important in deep learning, and we plan to release \sysname{} as open-source software to encourage further research on the optimization and algorithms on dynamic sparsity.

%% file: motivation_2.tex
\section{Background and Motivation}
\label{sec:motivation}

\begin{figure}[t]
  \centering
  % \captionsetup[subfigure]{labelformat=empty}
  \subfloat[Dynamic Attention]{
    \label{fig:dynamic_sparse_case_dynamic_algorithm}
    \includegraphics[height=2.8cm]{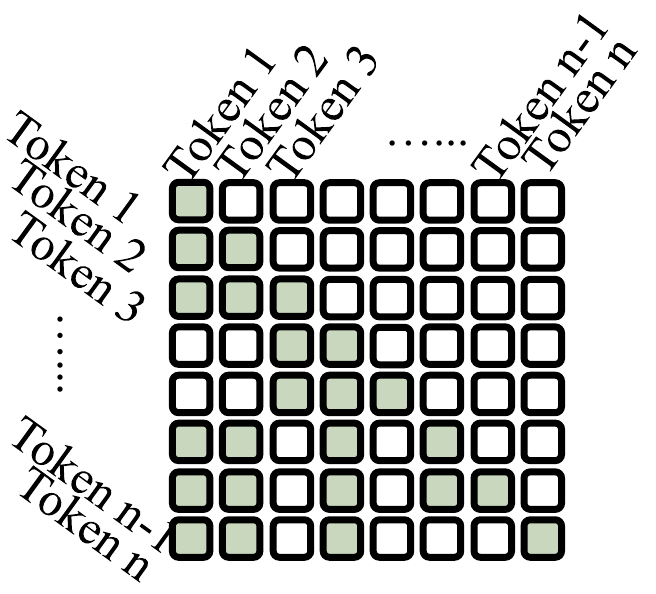}} 
  \subfloat[Mixture-of-Experts]{
    \label{fig:dynamic_sparse_case_moe}
    \includegraphics[height=2.5cm]{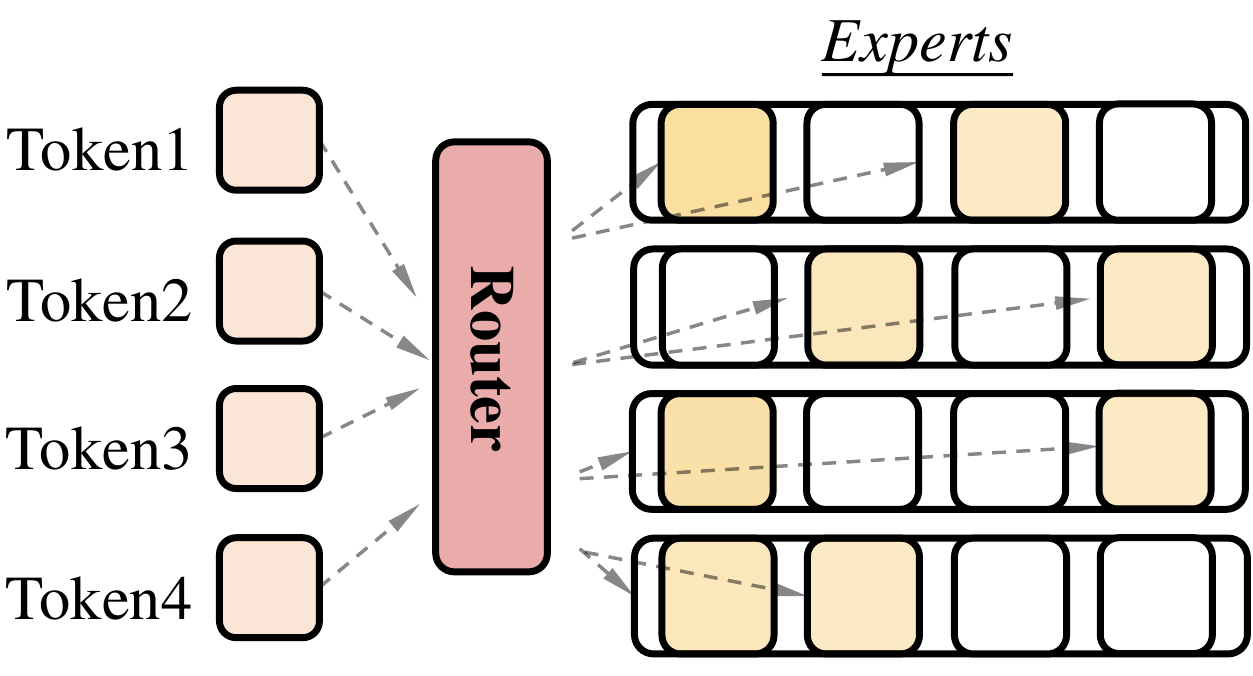}} \\
  \subfloat[Dynamic Sequence Length]{
    \label{fig:dynamic_sparse_case_seq_len}
    \includegraphics[height=2.3cm]{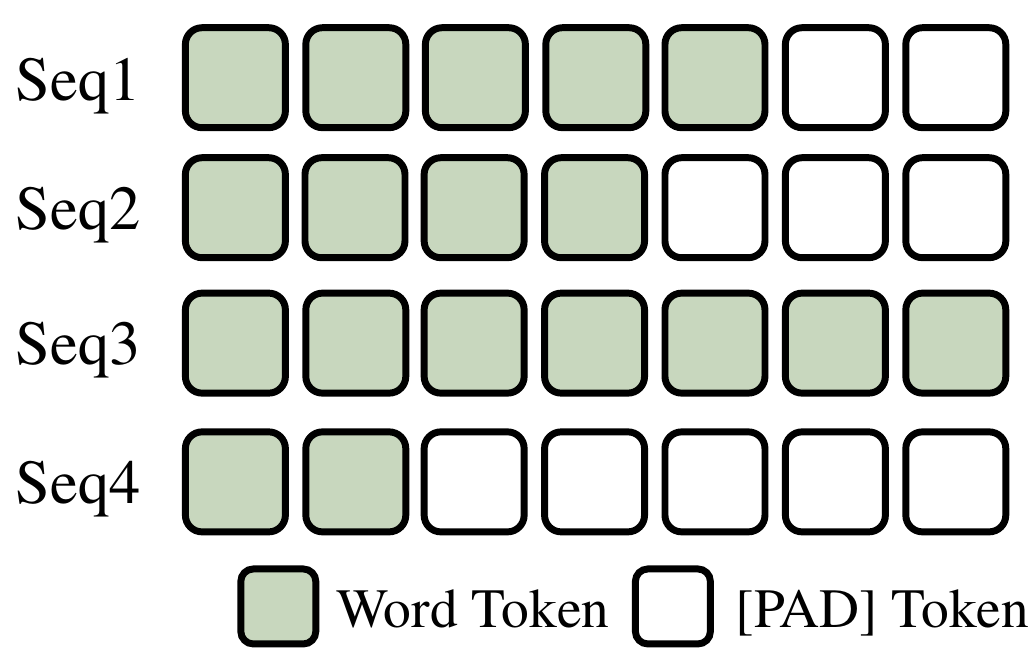}} 
  \subfloat[Sparse Training]{
    \label{fig:dynamic_sparse_case_dynamic_training}
    \includegraphics[height=2.2cm]{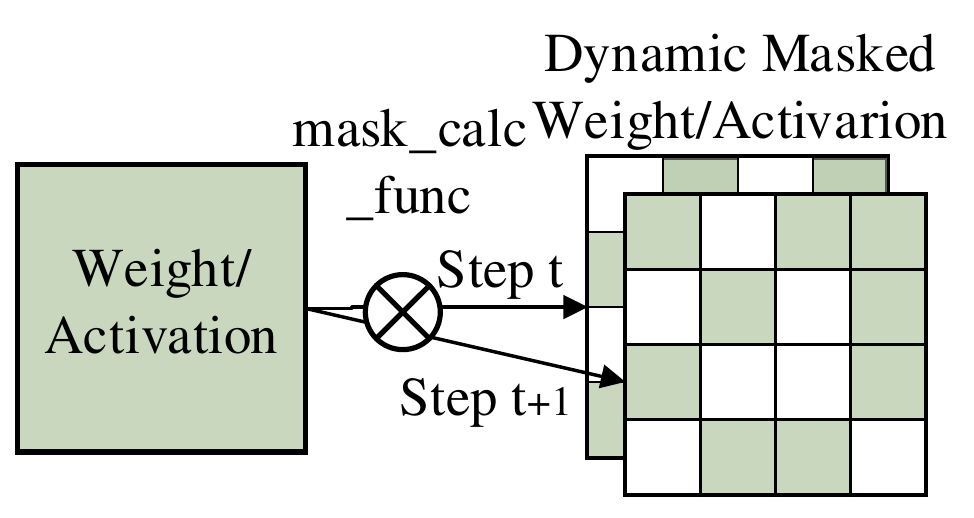}} 
  \caption{Examples of dynamic sparsity in deep learning.}
  \label{fig:dynamic_sparse_case}
\end{figure}

\subsection{Dynamic Sparsity}
%Deep learning models are becoming increasingly sparse over time. 
As deep learning models become increasingly deep and large, parameters, activations, and gradients tend to approach zeros or become exactly zero, resulting in model sparsity. %Especially, 
With the rapid development of Large Language Models (LLMs), it is a growing trend that many types of sparsity are shown to be dynamic, which is only known at runtime.

\para{Dynamic Sparsity in LLMs.}
Dynamic attention allows models to compute only on the most informative parts of the data~\cite{han2021dynamic,habibian2021skip,fedus2021switch,zoph2022designing,hwang2022tutel,wang2021spatten,zhou2022energon,rao2021dynamicvit,zhou2021informer,kitaev2020reformer,liu2021transformer}. As shown in \autoref{fig:dynamic_sparse_case_dynamic_algorithm}, models can mask out irrelevant tokens, achieving higher accuracy with less computation. The masks are generated on the fly. %by dedicated layers.
\revise{Our evaluation also} shows the activation outputs of three popular LLMs (OPT, Switch Transformer, and T5) have a sparsity ratio of 95-99.9\% (\ie percentage of zeros), which can be skipped without impacting the model accuracy (refer to \S\ref{sec:eval_end_2_end}).
% A large category of deep-learning models\cite{han2021dynamic,habibian2021skip,fedus2021switch,zoph2022designing,hwang2022tutel,wang2021spatten,zhou2022energon,rao2021dynamicvit,zhou2021informer,kitaev2020reformer,liu2021transformer} are dynamically attentive to the input data. Such dynamic attention allows the model to focus more on the informative region of the data to achieve better performance. \autoref{fig:dynamic_sparse_case_dynamic_algorithm} illustrates an example of dynamic sparse network\cite{habibian2021skip}. By masking out the irrelevant background, the model can achieve higher accuracy with less computation. The mask is online produced by several dedicated layers of the network.

\para{Sparse Training.}
Sparse training is critical and widely used in many fields of deep learning, such as Dynamic Sparse Training (DST)~\cite{Evci2019RiggingTL}, MAE~\cite{he2022masked}, pruning~\cite{nn_pruning}, and supernet training\cite{cai2019once, chen2021autoformer}. \autoref{fig:dynamic_sparse_case_dynamic_training} shows an example of dynamic pruning in training. The algorithm will mask out a portion of the weight according to the current model state. The mask in each step is constantly changing during training.

\para{Mixture-of-Experts (MoE).}
MoE is a neural architecture that shows dynamic sparsity and is widely used in popular large models. Most LLMs over one-trillion parameters adopt MoE to scale up the model ~\cite{fedus2021switch,lepikhin2020gshard,du2022glam}. As shown in \autoref{fig:dynamic_sparse_case_moe}, a sequence of tokens passes through a gating function that assigns each token to expert(s) dynamically. Each expert handles only a proportion of the tokens of its expertise, by masking the other tokens not routed to it. Therefore, each expert’s computation is sparse and depends on the input, which is only known at the runtime. 
%a sequence of tokens passes through a gating function (\ie Router) which assigns each token to an expert(s) in an online manner. Semantically, an expert only takes a proportion of the tokens that are located in its mastered domain (\eg two out of four in \autoref{fig:dynamic_sparse_case_moe}). Thus, the computation of each expert is dynamically sparse depending on the input tokens.

\para{Dynamic Sequence Length.}
In natural language processing~\cite{bert,ouyang2022training,raffel2020exploring} and multi-modality models~\cite{radford2021learning,ramesh2022hierarchical}, token sequences naturally have varied input and output lengths. Processing such sequences in batch requires padding them to the same sequence length (typically the maximum length in the batch), as shown in \autoref{fig:dynamic_sparse_case_seq_len}. Such padding leads to waste in computation and can be treated as dynamic sparsity.

\begin{figure}[t]
  \centering
  % \captionsetup[subfigure]{labelformat=empty}
  % \includegraphics[width=0.6\columnwidth]{Figure/experiments/motavtion_convert_computation.pdf}
  \subfloat[Latency and wasted computation of different tile sizes.]{
    \label{fig:trade_off}
   \includegraphics[height=4.2cm]{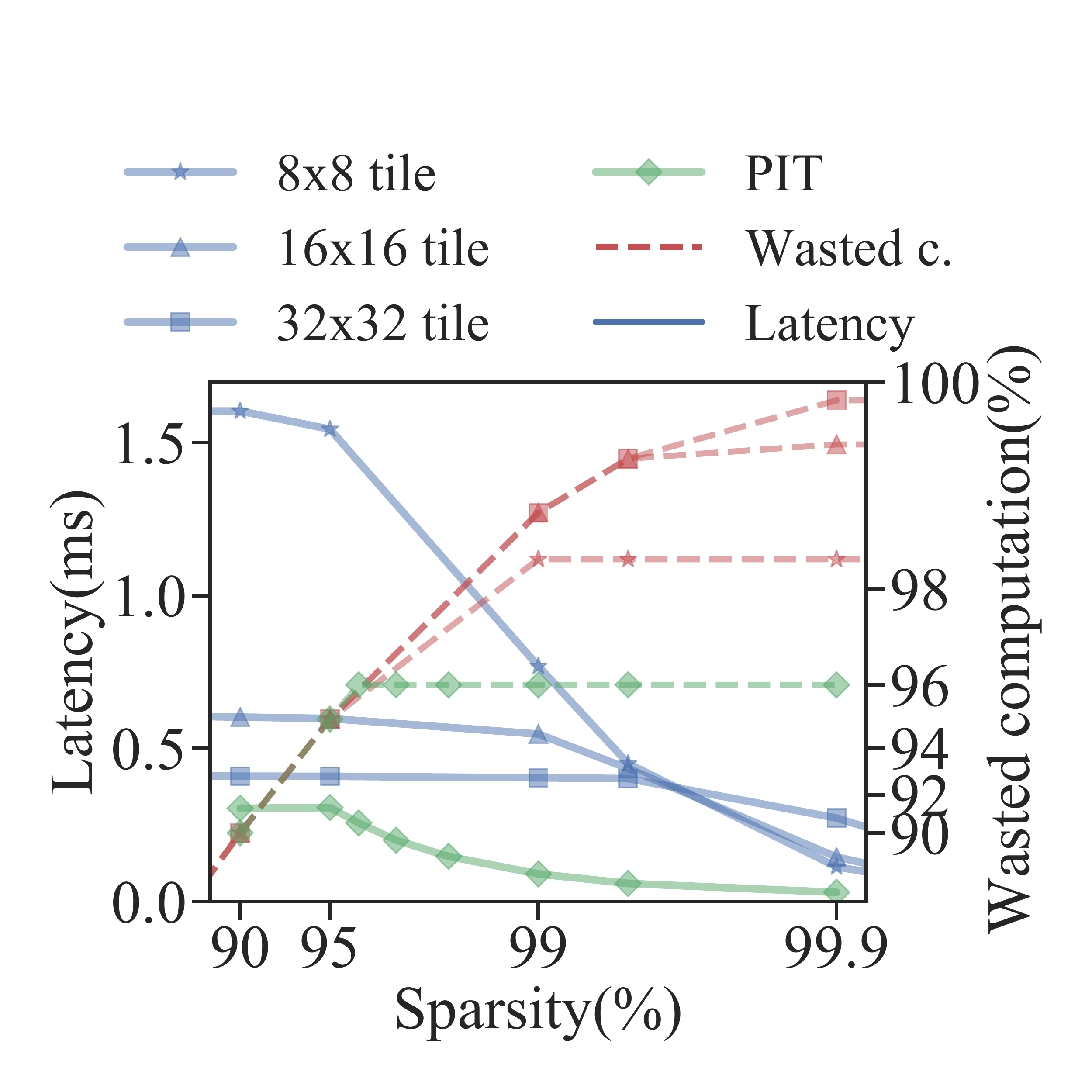}}
  \subfloat[Conversion overheads.]{
    \label{fig:convert_computation}
    \includegraphics[height=3.7cm]{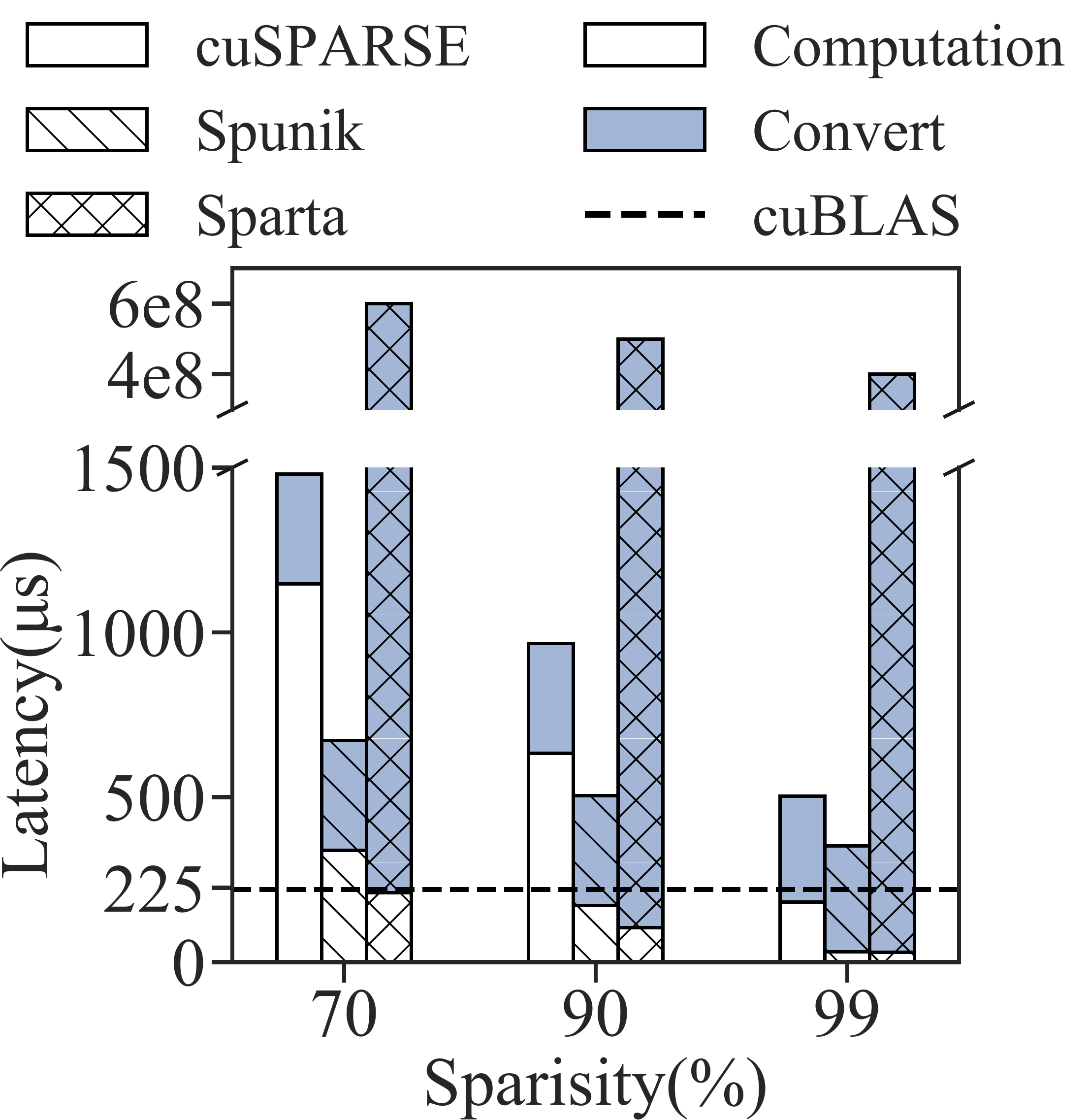}} 
  \label{fig:motivation_convert}
  \caption{Inefficient tiling of dynamic sparsity due to wasted computation and high sparse format conversion overhead.}
\end{figure}

%\subsection{Inefficient Tiling of Dynamic Sparsity}\label{dilemma}
\subsection{Inefficiency Due to Dynamic Sparsity}\label{dilemma}
\para{Tiling in deep learning compilers.}
Deep learning compilers like TVM~\cite{tvm}, Triton~\cite{triton}, Roller~\cite{roller}, often use tiling, a technique that slices a tensor into smaller tiles. By reusing cached tiles, tiling reduces the amount of data that needs to be transferred from slower memory like DRAM.
%For example, a large tensor can be divided into tiles of 32x32 elements. As the tensor is accessed for processing, the compiler will keep tile data in the cache and reuse it as possible, only loading new tile data when needed. This allows the program to skip re-loading parts of the tensor that were already cached, speeding up execution. 
%Tiling works especially well for tensors with spatial locality, where data is accessed in a pattern. 
By tuning tile shape (\eg 32x32 or 16x64), compilers can optimize data reuse for a particular model computation on a specified hardware architecture, thus improving kernel performance.

\para{Inefficient tiling in the presence of dynamic sparsity.}
%Although state-of-the-art DL compilers are shown good at efficient dense computation, they face a critical dilemma of inefficient tiling for dynamic sparsity patterns. 
Despite its effectiveness for conventional dense models, tiling can be inefficient for dynamic sparsity.
As shown in \autoref{fig:tile}, tile shapes aligned with the sparse pattern minimize the coverage of non-zero values but are inefficient when executed on GPUs. On the contrary, GPU-efficient tiles introduce waste (covering too many zeros). %due to their coarse granularity.
Using actual sparse activations extracted from OPT~\cite{zhang2022opt} (a large language model), \autoref{fig:trade_off} shows the performance of GPU kernels tuned with various tile shapes when executing a sparse matrix multiplication under different sparsity ratios. When the sparsity ratio is lower than 99.6\%, 32x32 tiles are more efficient although it contains the most wasted coverage on zeros. The 8x8 tiles are faster only when sparsity is very high (>99.9\% for this case) because of more saved computation than other tile shapes. Different tile shapes face the dilemma between sparsity-friendly tile coverage and GPU-efficient execution. %Is it possible to do sparse computation efficiently with both low waste and fast speed?

%Non-zero values appear as a block of 8x8. When there are 0\%-80\% zeros in the matrix, although the largest tile size (64x64) has much higher wasted computation on zeros, it still runs faster due to the higher GPU utilization. Only when the sparsity ratio is high enough (>80\%), 8x8 tiles can outperform the largest tile size due to saved computation. Moreover, when the non-zeros values randomly appear in the matrix, using a smaller tile size brings a marginal reduction in the wasted computation, until the tile size matches the granularity of sparsity.

%State-of-the-art
\para{Sparsity-aware compilers or libraries.} \revise{Some sparsity-aware compilers or libraries} try to break the dilemma by compiling specialized GPU kernels (\eg SparTA~\cite{sparta}), or transforming the data into a special format (\eg cuSPARSE~\cite{cusparse}, Spunik~\cite{sputnik}). They incur significant overheads that degrade runtime performance. %, which is not acceptable for dynamic sparsity. 
\autoref{fig:convert_computation} compares the conversion overhead (compiling or format transformation) of SparTA, cuSPARSE, and Spunik when handling dynamic sparsity. SparTA takes 400-600 seconds to compile the specialized kernel, impractical to handle sparsity patterns changed at the runtime. Although cuSPARSE and Spunik can be used for dynamic sparsity, they suffer from large transformation overheads. %($\sim330$us for this case), which is 
As a result, the overall performance is even worse than directly using dense computation when the sparsity ratio is high and thus is inefficient to handle dynamic sparsity.
% Execution of dynamic sparsity is inefficient on existing GPUs.
% 1. Large tile size v.s. waste. 
% 2. Small tile size, low waste but even worse than dense.
% 3. Drawbacks of DL compilers supporting dynamic sparsity: Cusparse/sputnik. high overhead.

\begin{figure}[t]
	\centering
	\includegraphics[width=0.95\columnwidth]{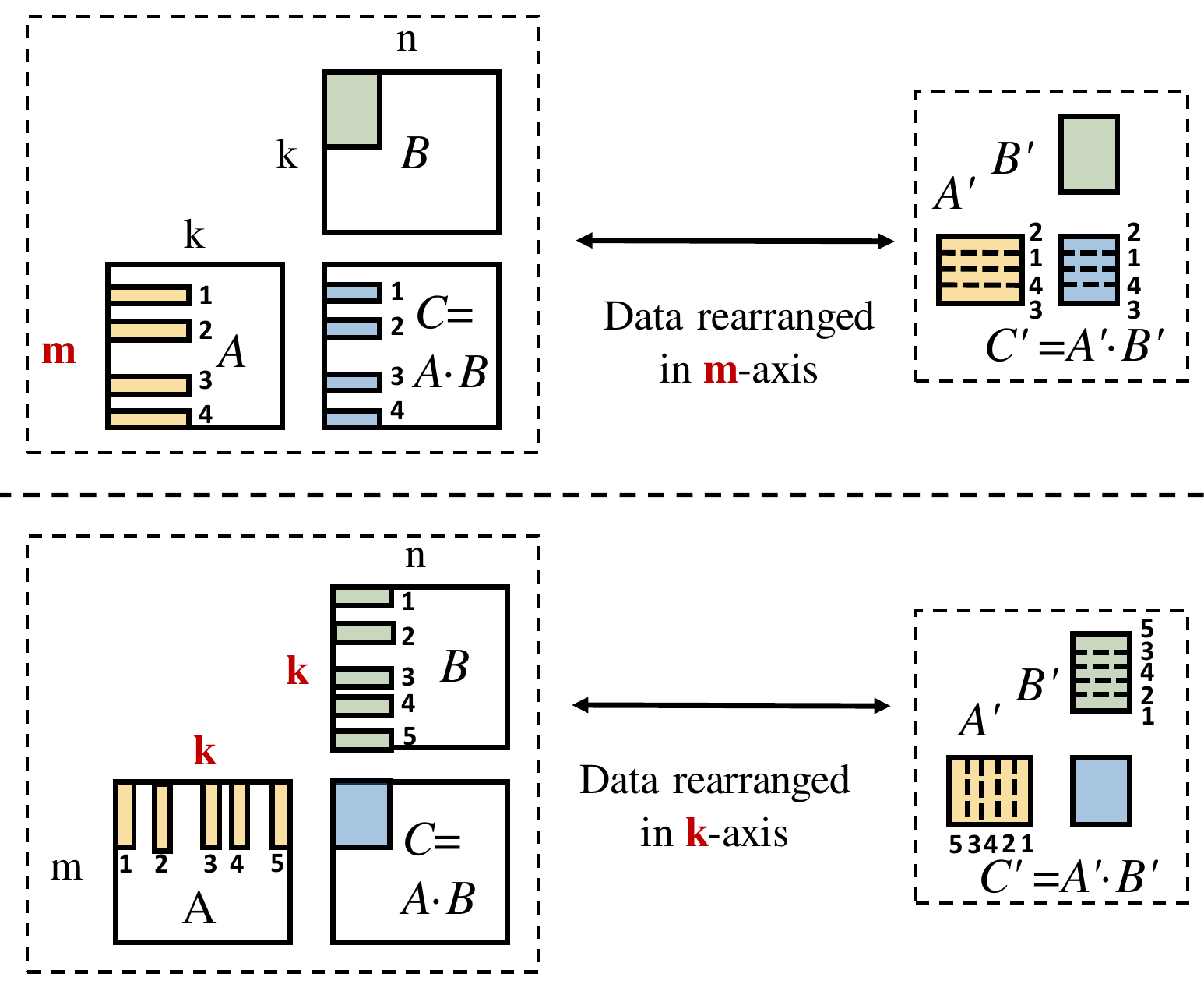}
    \vspace{0.1cm}
	\caption{\label{fig:matmul_example} Examples of sparse matrix multiplication. By rearranging sparse data along an axis, the sparse computation can be equivalently done with dense matrix multiplication.}
\end{figure}

\begin{figure*}[t]
	\centering
	\includegraphics[width=1.9\columnwidth]{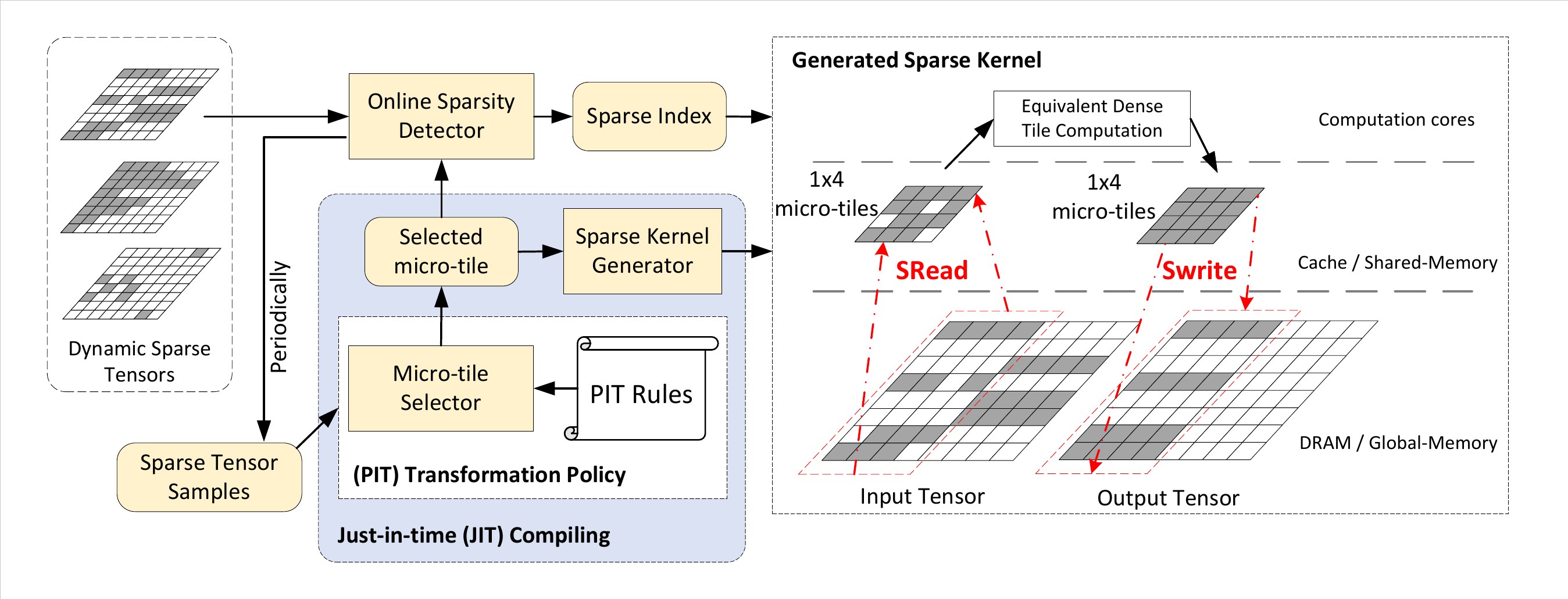}
	\caption{\label{fig:overview} Architecture overview: \sysname{} uses JIT compiling to generate sparse GPU kernels for dynamic sparse patterns. A sparse GPU kernel uses \sload{}  to load micro-tiles from the input tensor (detected at runtime) to dense computation tiles at the shared memory. After the equivalent dense tile computation, it uses \swrite{} to write output micro-tiles to the output tensor. \sload{} and \swrite{} are conducted only on PIT-axis, where \sysname{} guarantees the correctness of its data rearrangement.}
\end{figure*}

\subsection{Opportunity}
%Given that commodity accelerators are efficient at the dense computation of large tile sizes, can we leverage them for efficient dynamic sparsity computation? The dilemma of the inefficient tiling is that the tiles loaded into the shared memory must be a continuous slice of data from a tensor, however, we find it unnecessary. 
The right figure in \autoref{fig:tile} shows that it is possible to load data sparsely located at different positions in parallel, and merge them into a dense tile with a shape efficient for computation. This sparse-to-dense transformation will not affect the correctness of the result.  \autoref{fig:matmul_example} shows two examples of matrix multiplication. The first example multiplies a sparse tensor $A$ with a dense tensor $B$. The shaded area has non-zero values and the rest are zeros. By rearranging the non-zero rows of $A$ into a new tensor $A'$, the computation can be conducted on dense matrix multiplication between $A'$ and $B'$ (the non-zero dense tile of $B$) producing the result tensor $C'=A'\times B'$. After writing the rows of $C'$ to the original rows in C, we get the same result of $C=A\times B$. Note that, the rearrangement of $A$'s non-zero rows can be in any permutation without affecting the correctness of each row in $C'$. The rows of $C'$ are written back to $C$ with the reverse permutation of $A\xrightarrow{}A'$ to restore the correct row indexes. The second example shows the matrix multiplication of two sparse matrixes $A$ and $B$. By rearranging the data on the dimension $k$ of $A$ and $B$ (\ie columns of $A$ and rows of $B$), the calculation can be done similarly using the matrix multiplication between two dense matrixes (\ie $A'$ and $B'$). As we have shown in \autoref{fig:trade_off}, dense tiles are more GPU-efficient. If the data rearrangement in \autoref{fig:matmul_example} has negligible overhead at runtime, the dynamic sparse tensor computation can be conducted using dense tiles with low waste and high GPU efficiency, thus achieving superior performance. This motivates us to design \sysname{} to systematically exploit this \emph{permutation invariant transformation} of DL operators for efficient support to dynamic sparsity.

%% file: Design.tex
\section{\sysname{} Design} \label{sec:design}

%To tackle the aforementioned challenges, we have devised a novel dynamic sparsity compilation framework named \sysname{}. 
\sysname{} is a compiler framework designed to address challenges introduced by dynamic sparsity.
The core of \sysname{} is its permutation invariant transformation (abbr. PIT transformation), which converts sparse tensors into a computation-efficient dense format in an online manner. % with minimal overhead. 
\autoref{fig:overview} illustrates \sysname{}'s architecture.
\sysname{} introduces micro-tile, a data unit with a minimum size used to compose a larger, hardware-friendly tile for efficient computation. Given sparse tensors, the PIT transformation policy identifies the most efficient micro-tile from all feasible micro-tiles derived from PIT rules, which are mathematically equivalent computation transformations of deep learning operators. The sparse kernel generator then creates the sparse kernel based on the selected micro-tile. To handle dynamic sparsity, the kernel takes sparse data and the index of micro-tiles with non-zero values for proper computation. %elements at the granularity of the micro-tile. 
The sparsity detector constructs the index online. The PIT property enables full parallelism in sparsity and index construction, minimizing the online execution overhead. In the sparse kernel, \sload{} and \swrite{} rearrange the sparse data into a dense format based on the constructed index, which is then processed by the highly efficient dense tile-based computation.

\subsection{PIT Transformation Mechanism} \label{sec:pitmechanism}

\para{Micro-tile.}
%To achieve efficient execution of sparse operators, the key lies in transforming sparse computation into hardware-favored computation units, \ie dense tile in this paper. 
% \sysname{} abstracts hardware accelerators as multi-level hierarchical devices. Each level has local memory, higher level memory (\eg on-chip shared memory in an SM) is much faster than lower level memory (\eg GPU global memory). To efficiently process an operator, data is loaded from a lower level memory to a higher level memory and consumed by computing cores.
% Each level of a hierarchical device has its desirable data processing granularity. For instance, the granularity of global memory is the size of read/write transaction (\eg 32 bytes for CUDA GPU), while the granularity of computing cores in CUDA GPU is an SM. However, due to the diverse and irregular nature of sparsity patterns, aligning them with accelerator granularities is challenging. Therefore, we introduce the micro-tile abstraction to bridge the gap between sparse computation and accelerator's granularities.
Micro-tile is a small data unit with a shape aligned with
%aligns sparse data with 
the read/write transaction granularity of the lower level memory of an accelerator (\eg GPU). Micro-tile makes the access of sparse data as efficient as dense ones. For example, the read/write transaction of global memory in CUDA GPUs is 32 bytes, the smallest micro-tile size on this type of accelerator is 1x8 float32 (or 1x4 of float64), which is fine-grained enough for many sparse patterns. %The sparsely located non-zero values in a sparse tensor is covered with the micro-tile. 
For example, in \autoref{fig:overview}, the input tensor has non-zero values in a granularity of 1x1, 1x2, 1x3, and 1x4. This can be covered with 1x4 micro-tiles, assuming 1x4 is the size of read/write transaction. This way, micro-tile achieves a good trade-off between computation efficiency and coverage waste.

%To align with the granularity of an accelerator's other components, such as shared memory and computation cores, micro-tiles in a sparse tensor are mapped to a dense tile. 
An efficient dense tile has already been aligned with data access of GPU shared memory (\eg minimizing bank conflict~\cite{roller}) and saturated computation cores through well-optimized warp schedule~\cite{rammer}. By transforming the sparsely located micro-tiles to the required data format of the dense computation tile, the computation of a sparse operator is well aligned with every component of a GPU, including the global memory, the shared memory, and computation cores, thus achieving high efficiency. 
%It is worth noting that the micro-tile design can be based on any tile that executes efficiently on accelerators, such as a tile based on SparseTensorCores~\cite{}.
The transformation between sparse and dense computation tiles leverages permutation invariant transformation, a property that enables the dense tile to work on the rearranged micro-tiles correctly. Such a property commonly exists in deep learning operators, which will be elaborated in \S\ref{sec:pit_policy}.
\begin{figure} [t]
\begin{lstlisting}[language=Python]
class MicroTiledOp:
    # data format in global memory, A is sparse
    InputMicrotileSizes  # 1x4 for A, None for B
    OutputMicrotileSize  # None for C
    # data format in shared memory
    TileInputFormats     # 4x4 dense for A and B
    TileOutputFormat     # 4x4 dense for C
    DenseTileImpl        # C[4x4]=A[4x4]*B[4x4]
    def GenerateKernel()

\end{lstlisting}
\caption{The definition of micro-tiles for a sparse operator. The comments are the values of those attributes for the example micro-tile on the right of \autoref{fig:overview}.}
\label{fig:stile}
\end{figure}

% \begin{figure} [t]
% \centering
% \begin{minted}[
% breaklines=true,
% frame=lines,
% fontsize=\small,
% linenos=true,
% xleftmargin=2em,
% escapeinside=!!
% ]{python}
% class MicroTiledOp:
%     # data format in global memory, A is sparse
%     InputMicrotileSizes  # 1x4 for A, None for B
%     OutputMicrotileSize  # None for C
%     # data format in shared memory
%     TileInputFormats     # 4x4 dense for A and B
%     TileOutputFormat     # 4x4 dense for C
%     DenseTileImpl        # C[4x4]=A[4x4]*B[4x4]
%     def GenerateKernel()
% \end{minted}
% %\vspace{-0.5cm}
% \caption{The definition of micro-tiles for a sparse operator. The comments are the values of those attributes for the example micro-tile on the right of \autoref{fig:overview}.}
% \label{fig:stile}
% \end{figure}

\begin{figure} [t]
\centering
\begin{lstlisting}[language=c++]
/*Generated Sparse Kernel*/
__global__ void SparseKernelTemplate( 
    struct Tensor Inputs, struct SparseIdx InIdx,
    struct Tensor Output, struct SparseIdx OutIdx,
){ 
    /* First allocate shared memory */
    InTiBlocks = AllocSharedM(TileInputFormats);
    OutTiBlock = AllocSharedM(TileOutputFormat);
    SRead(Inputs, InTiBlocks, InIdx); 
    DenseTileImpl(InTiBlocks,OutTiBlock); 
    SWrite(OutTiBlock, Output, OutIdx);
}
\end{lstlisting}
% \vspace{-0.5cm}
\caption{The sparse kernel template with \sload{} and \swrite{}.}
\label{fig:stile_kernel}
\end{figure}

% \begin{figure} [t]
% \centering
% \begin{minted}[
% breaklines=true,
% frame=lines,
% fontsize=\small,
% linenos=true,
% xleftmargin=2em,
% escapeinside=!!
% ]{c}
% /*Generated Sparse Kernel*/
% __global__ void SparseKernelTemplate( !\label{line:sparsekernel}!
%     struct Tensor Inputs, struct SparseIdx InIdx,
%     struct Tensor Output, struct SparseIdx OutIdx,
% ){ 
%     /* First allocate shared memory */
%     InTiBlocks = AllocSharedM(TileInputFormats);
%     OutTiBlock = AllocSharedM(TileOutputFormat);
%     SRead(Inputs, InTiBlocks, InIdx); 
%     DenseTileImpl(InTiBlocks,OutTiBlock); !\label{line:densecomp}!
%     SWrite(OutTiBlock, Output, OutIdx);
% }
% \end{minted}
% % \vspace{-0.5cm}
% \caption{The sparse kernel template with \sload{} and \swrite{}.}
% \label{fig:stile_kernel}
% \end{figure}

\autoref{fig:stile} shows the definition of micro-tile on a sparse operator. It includes the micro-tile sizes for the operator's inputs/output and the dense computation tile, to which the micro-tiles are mapped. The attributes \texttt{TileInputFormats} and \texttt{TileOutputFormat} represent the data format (i.e., dense tile shapes) of the inputs and output respectively required by \texttt{DenseTileImpl}. 
%During online execution, the micro-tiles in sparse tensors are transformed to/from \texttt{TileInputFormats} and \texttt{TileOutputFormat}. 
A transformation policy (\S\ref{sec:pit_policy}) determines such information, used for sparse kernel generation later.
%The relationship between micro-tile and dense tile provides a clear structure for the kernel, as illustrated in Figure \autoref{fig:stile_kernel}. The kernel accepts sparse inputs in global memory and their sparse data indexes at the granularity of the specified micro-tile, which is constructed online by the sparsity detector. The two primitives, sload and swrite, rearrange the data based on the indexes.

\para{\sload{} and \swrite{}.}
In contrast to dense tensor computations, whose tiles are loaded, processed, and stored continuously, \sysname{} generates sparse kernels that employ \sload{} and \swrite{} to handle sparse data at the micro-tile level. %Specifically, \sload{} retrieves scattered micro-tiles from global memory and places them contiguously in shared memory. Conversely, \swrite{} writes computed results back to global memory at discrete addresses, as illustrated on the right of \autoref{fig:overview}. 
As illustrated on the right of \autoref{fig:overview}, two primitives \sload{} and \swrite{} do online rearrangement of micro-tiles in input tensors to prepare data in \texttt{TileInputFormats} and write the data in \texttt{TileOutputFormat} to output micro-tiles respectively. 
The data rearrangement is piggybacked on the data movement across different memory levels, resulting in little additional overheads and eliminating the need for traditional data rearrangement outside the sparse kernel (\eg constructing CSR format~\cite{bulucc2009parallel} for a sparse kernel).

\autoref{fig:stile_kernel} illustrates the template of the \sysname{}'s sparse kernel. The kernel consists of two phases: data arrangement using \sload{} and \swrite{}, and the computation on dense tiles. Both the data rearrangement of \sload{} and \swrite{} require fast online construction of micro-tiles' indexes (\ie \texttt{InIdx}, \texttt{OutIdx})  (\S\ref{sec:fast_index}). The indexes are constructed following the specified micro-tile shape defined in \autoref{fig:stile}. This design effectively separates data encoding/decoding from computation in the sparse kernel, introducing a novel sparse computation paradigm that combines data rearrangement with a (dense) computation tile.

%Since the computation unit is small, it is feasible to assemble different computation units within the same sparse kernel to handle more intricate sparse patterns and achieve greater efficiency. We leave it in future work.

% \begin{table}[t]
%     {\centering
%     \footnotesize{
%     \begin{tabular}{l|c}
%     \Xhline{2\arrayrulewidth}
%     \textbf{Operator} & \textbf{Tensor Expression} \\ \hline
%     ReduceSum & $C[p]$ += $A[p, l]$ \\
%     Vector Addition & $C[p]$ = $A[p]$ + $B[p]$ \\
%     MatMul & $C[m,n]$ += $A[m,k]$ * $B[k,n]$ \\
%     BatchMatmul & $C[b,m,n]$ += $A[b,m,k]$ * $B[b,k,n]$ \\
%     Convolution & $C[n,f,x,y]$ += $A[n,m,x+i,y+j]$ * $B[f,m,i,j]$ \\ \hline
%     \end{tabular}}
%     \par}
%     \caption{Example tensor expressions.\label{tab:expression}}
%      %\vspace{-0.5cm}
% \end{table}

\begin{table}[t]
\centering
    \footnotesize{
\begin{tabular}{l|l|l}
\hline
\textbf{Operator}        & \textbf{Tensor Expression}                                                                              & \textbf{PIT-axis} \\ \hline
ReduceSum       & $C[p]$ += $A[p,l]$                                                                               & $p,l$     \\ \hline
Vector Addition & $C[p]$ = $A[p]$+$B[p]$                                                                             & $p$       \\ \hline
MatMul          & $C[m,n]$ += $A[m,k]$*$B[k,n]$                                                                      & $m,n,k$   \\ \hline
BatchMatMul     & $C[b,m,n]$ += $A[b,m,k]$*$B[b,k,n]$                                                           & $b,m,n,k$ \\ \hline
Convolution     & \begin{tabular}[c]{@{}l@{}}$C[n,f,x,y]$ += \\$A[n,m,x+i,y+j]$*$B[f,m,i,j]$\end{tabular} & $n,m,f$   \\ \hline
\end{tabular}}\vspace{0.1cm}
\caption{Tensor expressions of widely-used operators and their PIT-axes that support shuffling their indexes without 
 affecting correctness.\label{tab:expression}}
\end{table}
\subsection{PIT Policy}\label{sec:pit_policy}
\revise{\sysname{} defines} a series of rules working on a certain tile axis that can correctly transform micro-tiles along this axis into GPU-efficient dense tiles. Specifically, a PIT rule contains the combination of a PIT-axis, a micro-tile shape, and a dense computation tile. Following a PIT rule, the system applies \sload{}/\swrite{} on the PIT-axis, loading/writing multiple sparsely located micro-tiles on this axis into/from the dense computation tile. A PIT rule ensures the computation on the PIX-axis must satisfy the permutation invariant property~\cite{lee2019set}. %Finally, the GPU kernel of dense tiles is used for efficient computation.

For ease of exposition, we use Einstein summation (einsum) notation~\cite{Einstein} to express operations along tensor axes. \autoref{tab:expression} lists some common operators in deep learning and their corresponding einsum notations. An axis of an einsum notation is PIT-axis if and only if any shuffling of  indexes on this axis does not affect the correctness of the operator. The following theorem finds all PIT-axes of an operator.
 % In simplified einsum notations, there are two types of axes: the spatial axis, which appears in the output tensor, and the reduction axis, which is absent from the output tensor. The following theorem identifies the axes that are permutation invariant.
\begin{theorem}\label{theorem:commutative}
    An axis is called PIT-axis, if and only if all computations on the axis are commutative and associative.
\end{theorem} 
\noindent The theorem is mostly self-evident, and we omit the proof due to space constraints. The commutative and associative property guarantees the correctness of random shuffling of micro-tiles on the PIT-axis and allows the parallel processing of micro-tiles in any order. 
First of all, the axes that derive new axes are not PIT-axes. \revise{For example, axes of $x$, $i$, $y$, and $j$ in the convolution operator (\autoref{tab:expression}) are not PIT-axes because they are not commutative due to the new axes (``$x+i$'' and ``$y+j$'') derived by them. In the rest axes, there are two types of axes that are PIT-axes. }
An axis in the output tensor is called a \emph{spatial axis}. All spatial axes are PIT-axes since they only change the data layout. 
An axis not in the output tensor is called a \emph{reduction axis}, whose computations are mostly commutative and associative (\eg sum, multiply, max, min), and thus PIT-axes. 
\autoref{tab:expression} also summarizes the PIT-axes of the listed operators. For every operator used by a model, \sysname{} uses its tensor expression to find all PIT-axes based on the type of an axis, and whether it involves non-commutative or non-associative computations. We find 
 only using one PIT-axis is general enough to cover most dynamic sparsity in deep learning. Although we do identify more complex PIT rules, \eg permutation over multiple axes in $(b,m)$-axes or $(b,n)$-axes in BatchMatMul, we leave them to future works due to the limited space.

\para{Micro-tile and Kernel Selection.} %\label{sec:selectionpolicy}
% \para{Optimizing \stile{} kernel.}
%The efficiency of the sparse kernel generated by \sysname{} is mainly determined by two factors: the efficiency of the dense kernel before applying \sload{} and \swrite{}, and the efficiency of converting dynamic sparsity in the input tensor to the target dense kernel. The computational efficiency of the dense kernel is mainly determined by the size of the computation tile and the access pattern of shared memory. A smaller computation tile may not fully saturate the computing resources. Carefully designed thread coordination can achieve data reading and writing without bank conflict on shared memory, greatly improving the efficiency of dense tiles. For different dense tiles, \sysname{} pre-builds a performance lookup table as a reference for selecting the dense kernel (see \section{implementation} for details).
\sysname{} creates a database of sparse kernels, each of which applies PIT transformations on one PIT-axis of an operator. % Each generated sparse kernel has a corresponding micro-tile granularity. 
Each sparse kernel defines a micro-tile shape, which is determined by the PIT-axis and the tile shape of the dense kernel it uses for computation. 
%The shape of a micro-tile is determined by the permutation invariant axis and hardware characteristics. 
When the memory layout of the sparse tensor is not contiguous on the PIT-axis, we set the shape of micro-tiles to 1 on the PIT-axis while keeping the shape of other axes the same as the tile shape of the dense kernel. This allows GPUs to load/write the values along the PIT-axis in parallel which can saturate the memory transaction. For example, consider a dense matrix-multiplication kernel with a tile size of $[M, K]\times[K, N]$, suppose the first input tensor is sparse and stored in the row-major memory layout (\ie contiguous on $K$-axis in memory). If $M$ is the PIT-axis, the micro-tile size will be $[1, K]$. If the memory layout of the sparse tensor is contiguous on the PIT-axis, we need to first change its format to make the data non-contiguous on the PIT-axis, \eg from row-major to column-major, to saturate the memory transaction. This can be done in a piggyback manner at the output of previous operators generating this sparse tensor, thus its overhead is negligible.

\begin{algorithm}[t]
\caption{Kernel selection for a dynamic sparsity operator.} \label{alg:selection}
\KwData{$Op$: A dynamically sparse operator,\newline
        $D_{sparse}$: A list of $n$ sparsity samples of $Op$.}
\KwResult{$Best$: The best computation tile for $Op$.}
\SetKwProg{Fn}{Function}{:}{}
% \Fn{DeriveMicroTile($T_S$, $A$)}{\label{alg:derive_micro_tile}
%     $micro\_tile$ = $T_S.TileSize$\;
%     \If{!IsDataContinuousOnAxis($T_{S}$, $A$)}{
%         $micro\_tile[A]$ = 1\;
%     }\Else{
%         $micro\_tile[A]$ = $MemTransSize/sizeof(T_{S}.type)$\;\label{line:micro_tile_contiguous}
%     }
%     \KwRet $micro\_tile$\;
% }

\Fn{KernelSelection($D_{sparse}$, $Op$)}{ \label{alg:dynamic_opt}
    $Best$ = null; $Cost_{optimal}$ = inf\;
    \ForEach{$T \in$ GetTilesFromTileDB($Op$)}{ \label{line:get_tiles}
        \ForEach{$A \in$ GetPITAxis($Op$)}{ \label{line:get_pit_axes}
            $Cost$ = 0\;
            $micro\_tile$ = $GetMicroTile(T.SparseTensor, A)$\;  \label{line:get_micro_tile}
            \ForEach{$D \in D_{sparse}$}{
                $Num_{tiles}$ = CoverAlgo($D$, $micro\_tile$, $A$)\; \label{line:coveralgo}
                $Cost$ += $Num_{tiles} * T.tile\_cost$\;
            }
            \If{$Cost$ < $Cost_{optimal}$}{
                $Best$ = $S$\;
                $Cost_{optimal}$ = $Cost$\;
            }
        }
    }
    \KwRet $Best$\;
}
\end{algorithm}
 \autoref{alg:selection} shows how \sysname{} selects the appropriate PIT-axis, micro-tile shape, and dense computation tile to generate the sparse kernels in a JIT manner.
 \sysname{} iterates through all dense computation tiles and the PIT-axes of the operator (\autoref{line:get_tiles}-\autoref{line:get_pit_axes}). \texttt{GetTilesFromTileDB} returns all possible dense computation tile shapes with efficient GPU kernels, which are usually provided by existing DL compilers or implementation, \eg TVM~\cite{tvm}, OpenAI Block Sparse~\cite{openai_block}. \texttt{GetPITAxis} returns all feasible PIT-axes of the operator we defined in \autoref{theorem:commutative}. For each dense computation tile and PIX-axis, 
 \texttt{GetMicroTile} finds the valid micro-tile shape as we elaborated above (\autoref{line:get_micro_tile}). \texttt{CoverAlgo} (\autoref{line:coveralgo}) calculates the number of micro-tiles required to cover all non-zero values in the sparse tensor (denoted as $Num_{tiles}$). The time cost of the generated sparse kernel is estimated as $Num_{tiles}\times T.tile\_cost$, where $T.tile\_cost$ is the running time of the corresponding sparse kernel via offline profiling. \sysname{} selects the sparse kernel with the lowest time cost and its corresponding micro-tile to perform sparse computation. 
 \revise{As for the input whose sparsity ratio is relatively low, this kernel selection algorithm makes \sysname{} seamlessly fall back to the dense computation.}

 \revise{The offline profiling for the tile cost is lightweight. As \sysname{} chooses to merge micro-tiles into a dense tile along a certain PIT-axis on-the-fly, it allows the offline profiling to be done in \textit{a model, tensor shape, and sparsity pattern agnostic way}. \sysname{} just records the execution time of different tile shapes (\eg 32x32 and 64x64) for dense computation. Therefore, the offline profiling is conducted once per operator and per GPU type, which is very lightweight compared to long-running inference services.}

\subsection{Online Sparsity Detection}\label{sec:fast_index}
%\sysname{} to handle various sparse patterns by generating appropriate indexes for \sload{} and \swrite{} based on input and output.
%With the selected micro-tile, \sysname{} detects non-zero values in the granularity of the micro-tile.
\revise{Efficient sparse computation requires the online detection of the changing sparsity pattern and computes only the non-zero values in sparse tensors. This implies that the index for the non-zero values should be constructed on-the-fly. %In \sysname{}, \sload{} and \swrite{} use this index to load and rearrange the non-zero values. 
However, creating sparse indexes on-the-fly can be challenging. As illustrated in \autoref{fig:convert_computation}, cuSPARSE's conversion overhead can be significantly higher than computation time, particularly when sparsity is high.}

\revise{%For the efficiency of online sparsity detection, we 
We propose an effective mechanism for online index construction, \ie constructing a sparse index at the granularity of \textit{micro-tile} in an \textit{unordered} manner. First, with the selected micro-tile, \sysname{} detects non-zero values at the granularity of that micro-tile, which greatly reduces the size of the sparse index. Second, the PIT transformation enables sparsity detection and index construction to be highly concurrent in an out-of-order manner: the PIT-axis allows \sysname{} to perform computations when an axis is permuted (\S\ref{sec:pit_policy}). With this transformation, \sysname{} no longer needs an ordered index along a specific axis. This substantially reduces constraints during the micro-tile based index construction, minimizing synchronization overheads across accelerator threads.}

\revise{Specifically, the index construction task runs on the accelerator (\eg GPU) with the construction task divided into tiles. 
For each tile, the task traverses a region of the sparse tensor and checks for the micro-tiles containing non-zero values. When a non-zero micro-tile is detected, its index (i.e., the offset within the tensor) is written to a pre-allocated index array. 
As multiple tiles concurrently update the index array with indexes of non-zero micro-tiles, they use \texttt{atomicadd} to determine unique positions in the array for recording these indexes, ensuring the safety of the update. %unique positions in the array for recording these indexes. 
Consequently, the resulting index arrangement for micro-tiles is unordered due to the unpredictable scheduling order of thread blocks.}

\revise{Moreover}, \revise{unlike existing sparse solutions (\eg cuSPARSE, MegaBlocks),
\sysname{} constructs} sparse indexes without changing the storage format of sparse tensors. At runtime, \revise{\sload{} and \swrite{} in sparse kernels use the index to load and rearrange the non-zero values directly from and to the original sparse tensors, at the micro-tile granularity}. This significantly reduces memory access overheads introduced by data format conversion (\eg from sparse data in a dense tensor to CSR format) and achieves zero-copy data rearrangement.

%% file: Implementation.tex
\section{Implementation}\label{sec:implementation}

We implement \sysname{} and integrate it with PyTorch~\cite{pytorch}. It consists of approximately 13,000 lines of C++ and CUDA code, and 5,600 lines of Python. PyTorch is a popular open-source DNN framework that supports various dynamic sparsity algorithms. %In particular, \sysname{} features a just-in-time code generator, a rapid online sparse index constructor, online micro tile selection policies.
%\sysname{} has summarized the PIT transformation rules for common deep learning operators, based on the transformation policy outlined in \autoref{sec:pit_policy}. Furthermore, 
\sysname{} has generated approximately 1,500 sparse kernels by applying these PIT transformation rules to over 500 dense computation kernels, which include manually optimized kernels (such as OpenAI Block Sparse~\cite{openai_block}), hardware instruction accelerated kernels (\ie wmma~\cite{wmma}), and dense kernels generated by compilers like TVM~\cite{tvm}. These sparse kernels are stored in a database, and a performance look-up table is created in advance. %The performance %To evaluate the performance of each generated sparse kernel, its kernel latency is measured with zero sparsity and divided by the number of remaining micro-tiles. 
This profiled performance is then used to guide online micro-tile selection. Although offline profiling takes several hours, it is done only once and can be accelerated by parallel profiling on multiple devices.

With the extensibility of dynamic sparsity optimizations, \sysname{} have supported 31 Natural Language Processing (NLP) models. Thirteen (13) of them are large language models including OPT, T5, and Switch Transformer. Additionally, \sysname{} supports 8 dynamic sparse attention models, 8 MoE models, and 13 sparse training algorithms. Integration of \sysname{} is facilitated by its ability to accommodate minimal code modifications - with less than 10 lines of code changed in all evaluated scenarios. This feature enables swift adaptation of existing models for dynamic sparsity optimization, providing users with enhanced efficiency.

%% file: Evaluation.tex
\section{Evaluation}\label{sec:eval}

\input{Figure/experiments/inference_models}

%We evaluate \sysname{} with five representative models on different precisions (float32 and float16) and devices (A100 and V100). 
\revise{In this section we present comprehensive experiments to demonstrate the effectiveness of \sysname{} from various perspectives. Specifically, we first evaluate the end-to-end inference performance of \sysname{} with six representative models on both A100 and V100 (shown in \S\ref{sec:eval_end_2_end}). We also show the end-to-end performance of \sysname{} across two distinct training scenarios (detailed in \S\ref{sec:eval_end_2_end_train}). Furthermore, micro-benchmarks are performed to highlight the effectiveness of PIT transformation (\S\ref{sec:eval_effect_stile}) and the conversion overheads (\S\ref{sec:eval_convert}). Finally, we show the effectiveness of micro-tile online searching (\S\ref{sec:eval_online_search}) and detection of changing sparsity patterns (\S\ref{sec:eval_dynamic_sparse_pattern}). In summary,} our results show that: 

\begin{figure*}[t]
  \centering
  % \captionsetup[subfigure]{labelformat=empty}
  \subfloat[Latency]{
    \label{fig:exp_moe_latency}
    \includegraphics[width=2\columnwidth]{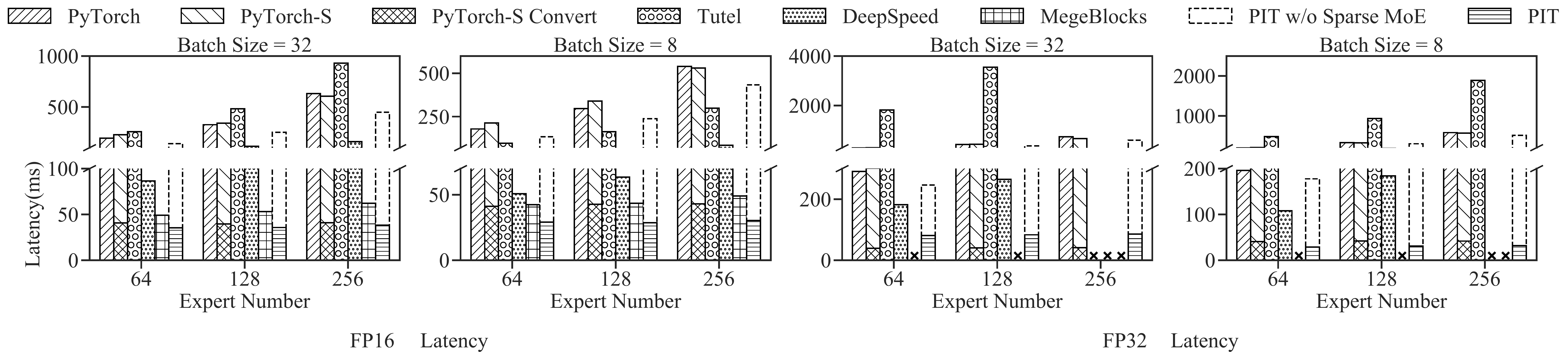}} \\
  % \subfloat[FP32 Latency]{
  %   \label{fig:exp_moe_latency_fp32}
  %   \includegraphics[width=0.98\columnwidth]{Figure/experiments/speedup_comparation_moe_fp32.pdf}} \\
  \subfloat[GPU Memory]{
    \label{fig:exp_moe_mem}
    \includegraphics[width=2\columnwidth]{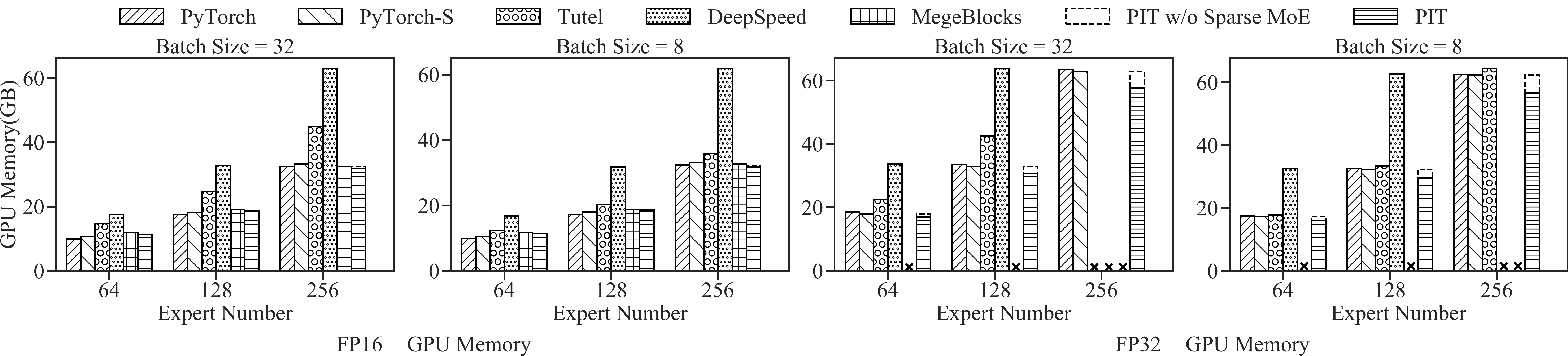}}
  % \subfloat[FP32 Memory]{
  %   \label{fig:exp_moe_mem_fp32}
  %   \includegraphics[width=0.98\columnwidth]{Figure/experiments/memory_comparation_moe_fp32.pdf}}
\vspace{0.2cm}
  \caption{End-to-end latency per batch and memory footprints of Switch Transformer.}
  \label{fig:exp_moe}
\end{figure*}

\begin{itemize}[leftmargin=*]
    \item \sysname{} achieves significant inference latency reduction and smaller memory footprints on five representative models, outperforming PyTorch, PyTorch sparse\footnote{The state-of-the-art sparse kernels wrapped in PyTorch.}, Tutel, DeepSpeed, MegaBlocks, and TurboTransformer by up to 18.1x, 17.8x, 59.1x, 5.9x, 1.6x, and 1.9x respectively (\S\ref{sec:eval_end_2_end}).
    \item \sysname{} also boosts the training efficiency significantly. Compared to the state-of-art solutions, PIT achieves up to 1.8x speedup for the OPT training, and up to 2.4x speedup for the sparse training (\S\ref{sec:eval_end_2_end_train}).
    \item With PIT transformation, \sysname{} outperforms the state-of-art sparsity optimizations. Specifically, \sysname{} achieves up to 88.7x, 5.8x, 17.5x, and 5.7x speedup over cuSPRARSE, Sputnik, OpenAI Block Sparse, SparTA respectively (\S\ref{sec:eval_effect_stile}).
    \item \revise{\sysname{} can detect dynamic sparsity online with negligible overheads and achieves up to 4.7x speedup over previous state-of-art works when constructing the sparse index online (\S\ref{sec:eval_convert}). }
\end{itemize}

\subsection{End-to-End inference}
\label{sec:eval_end_2_end}

We compare \sysname{} with state-of-the-art dense and sparse baselines on inference latency and memory usage for \revise{six} representative models as shown in \autoref{tab:inference_models}. The baselines include the most popular deep learning framework (PyTorch v1.11.0), two inference frameworks optimized for large-scale models (DeepSpeed~\cite{deepspeed} and TurboTransformers~\cite{turbo}), and model-specific optimization techniques (Tutel and MegaBlocks~\cite{megablocks} for MoE models and Longformer-S for the Longformer model). We also create PyTorch-S, a variant of PyTorch that uses the best-performing sparse kernels from cuSPARSE (v11.6)~\cite{cusparse}, Sputnik~\cite{sputnik}, and Triton~\cite{triton}. We select the best result among these sparse kernels for each model as the final performance of PyTorch-S. TurboTransformers only supports the BERT model and fails to run other models due to missing operators.

\para{Switch Transformer.}
%We evaluate \sysname{} on a representative MoE model, specifically the Switch-Transformer\cite{fedus2021switch}, using NVIDIA A100 with half-precision(float16). We compare \sysname{} with PyTorch, Tutel, DeepSpeed, and MegaBlocks. In the MoE model, each token is assigned to an expert and we find that the distribution of the number of tokens routed to each expert is imbalanced. PyTorch executes all experts serially and Tutel and DeepSpeed leverage BatchMatmul to horizontally fuse all experts so that all experts are executed in parallel. However, BatchMatmul requires the same number of received tokens for all experts, which may lead to wasted computation. MegaBlocks will reorganize the tokens in the sparse format and leverage sparse kernel to perform all experts simultaneously.
We evaluated the performance of \sysname{} on Switch Transformer, a large language model \revise{that consists of Encoder, Decoder, and MoE structure at the same time.} We measured the latency and memory usage of \sysname{} on 1x A100-80GB GPU with different precisions, namely float32 and float16. The MoE layer of Switch Transformer assigns each token to one of the experts, which may produce an uneven token distribution among the experts.
%Switch Transformer\cite{fedus2021switch} is a large language model that includes Encoder, Decoder, and MoE structure. We evaluated \sysname{} on Switch Transformer using A100 for latency and memory consumption at different precisions (float32 and float16).
%In the MoE layer, each token is assigned to an expert, leading to an imbalanced distribution of tokens across experts. 
We compare \sysname{} with several baselines: including PyTorch, which executes experts sequentially; Tutel and DeepSpeed, which use BatchMatmul to fuse all experts for parallel execution; MegaBlocks, which leverages sparse kernels to execute all experts simultaneously after reorganizing tokens in a sparse format. \revise{MegaBlocks only provides GPU kernels of float16 precision and thus is not evaluated in float32.}

\begin{figure*}[t]
  \centering
  % \captionsetup[subfigure]{labelformat=empty}
  \includegraphics[width=2\columnwidth]{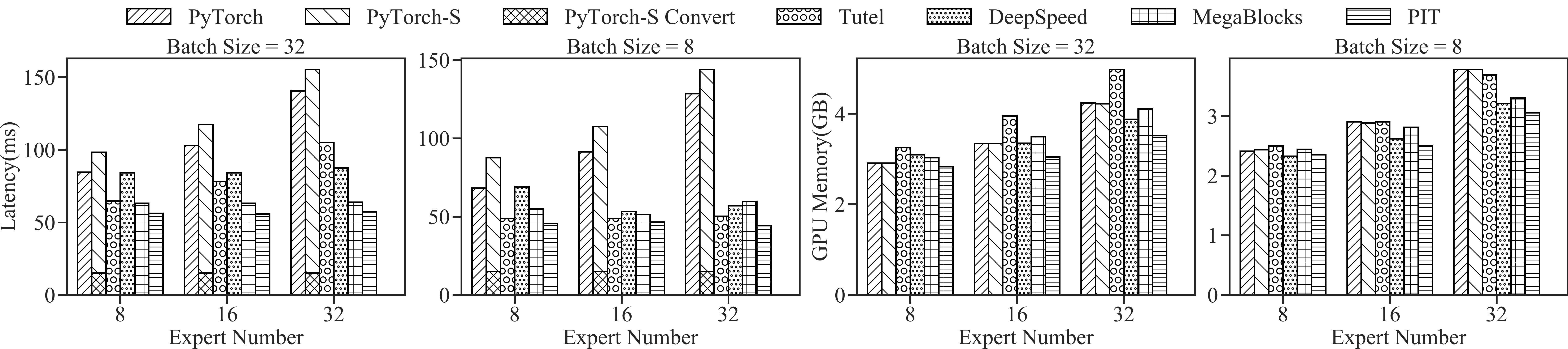}
  % \subfloat[Latency]{
    % \label{fig:exp_seq_lat}
  %   \includegraphics[width=0.98\columnwidth]{Figure/experiments/speedup_comparation_seq_len_glue.pdf}} 
  %   % \hspace{4em}
  % \subfloat[Memory]{
  %   \label{fig:exp_seq_mem}
  %   \includegraphics[width=0.98\columnwidth]{Figure/experiments/memory_comparation_seq_len_glue.pdf}} \\
\vspace{0.1cm}
  \caption{\revise{End-to-end latency and memory footprints of SwinMoE on A100.} } 
  \label{fig:exp_swinmoe}
\end{figure*}

In the non-MoE layers, \sysname{} optimizes dynamic sparsity caused by varying sequence lengths within the same batch. In the MoE layers, \sysname{} employs \sload{} to load the relevant tokens for each expert, sparsely computes their assigned tokens, and writes the results directly to the corresponding positions using \swrite{}.
We evaluated \sysname{} using the MNLI dataset in GLUE~\cite{wang2018glue}. \autoref{fig:exp_moe} shows end-to-end inference latency and memory cost of Switch Transformer with varying batch sizes and numbers of experts. As shown in \autoref{fig:exp_moe_latency}, %MegaBlocks %must be a typo, right?
\sysname{} outperformed other methods by allowing sparse calculation for all experts without computation waste. Compared to PyTorch, PyTorch-S, Tutel, DeepSpeed, \sysname{} achieved 3.6x$\sim$18.1x, 3.7x$\sim$17.8x, 16.6x$\sim$59.1x, 2.3x$\sim$5.9x speedup respectively when the precision is float32. For the precision of float16, the speedup was 5.5x$\sim$17.8x, 6.4x$\sim$17.5x, 3.3x$\sim$24.3x, 1.8x$\sim$4.2x, and 1.4x$\sim$1.7x compared to PyTorch, PyTorch-S, Tutel, DeepSpeed, and MegaBlocks, respectively. \revise{Compared to Tutel and DeepSpeed, \sysname{} avoids computational waste brought by the BatchMatmul, which requires padding the input of all experts to the same length.} 
\revise{Moreover, compared to MegaBlocks, \sysname{} uses  \sload{} and \swrite{} to eliminate the expensive data reorganization cost during input preparation.} ``PyTorch-S Convert'' highlights the sparse index construction overhead of PyTorch-S. Even though the computation in PyTorch-S has become faster, the cost of constructing sparse indices has neutralized the speed gains. \revise{To dissect the benefit under sparse MoE and varying sequence length, we also evaluated \sysname{} without applying dynamic sparse MoE optimization (\ie ``PIT w/o Sparse MoE'' in \autoref{fig:exp_moe}}. \revise{We find \sysname{}'s performance gain on Switch Transformer mainly comes from optimizing the dynamic sparsity in the MoE structure.}
We also evaluated the GPU memory usage shown in \autoref{fig:exp_moe_mem}. \sysname{} has the lowest memory usage compared to the baselines. \revise{Due to excessive padding when increasing the batch size and number of experts, Tutel and DeepSpeed run into Out-of-Memory (OOM).} % MegaBlocks is only capable of float16 optimization and cannot handle float32.

\para{\revise{Swin-MoE}}
\revise{We assessed the performance of \sysname{} on Swin-MoE, a large vision model with both Encoder and MoE structures. We measure the latency and memory usage of \sysname{} on A100 under the float16 precision. For vision transformers, the input images within the same batch will be rescaled to the same resolution to achieve a consistent sequence length. In our experiments, shown in \autoref{fig:exp_swinmoe}, we compared Swin-MoE's end-to-end inference latency and memory footprints across different batch sizes and numbers of experts. % As shown in Switch Transformer, \sysname{} optimizes the dynamic sparsity of MoE. 
MegaBlocks outperforms other baselines due to its simultaneous execution of all experts, efficiently utilizing sparse kernels to avoid computational waste. Compared to MegaBlocks, \sysname{} further improves the performance by piggy-backing data reorganizations in the data movement across memory hierarchies. Compared to PyTorch, PyTorch-S, Tutel, DeepSpeed, MegaBlocks, \sysname{} achieves 1.5x$\sim$6.3x, 1.5x$\sim$2.9x, 1.1x$\sim$1.8x, 1.2x$\sim$1.6x, 1.1x$\sim$1.4x speedup, respectively. \sysname{}'s performance improvement for Swin-MoE is less than that for Switch Transformer because the number of experts in Swin-MoE is significantly fewer than that in Switch-Transformer. As a result, the MoE layers only contribute 23.6\% to 61.2\% of the end-to-end latency when the number of experts varies from 8 to 32. When comparing the latency of the MoE layers alone, PIT is approximately 1.2x$\sim$1.7x faster than Megablocks. For a similar reason, \sysname{} has a similar GPU memory usage compared to the baselines as shown on the right side of \autoref{fig:exp_swinmoe}. }

\begin{figure}
  \centering
  % \captionsetup[subfigure]{labelformat=empty}
  \includegraphics[width=1\columnwidth]{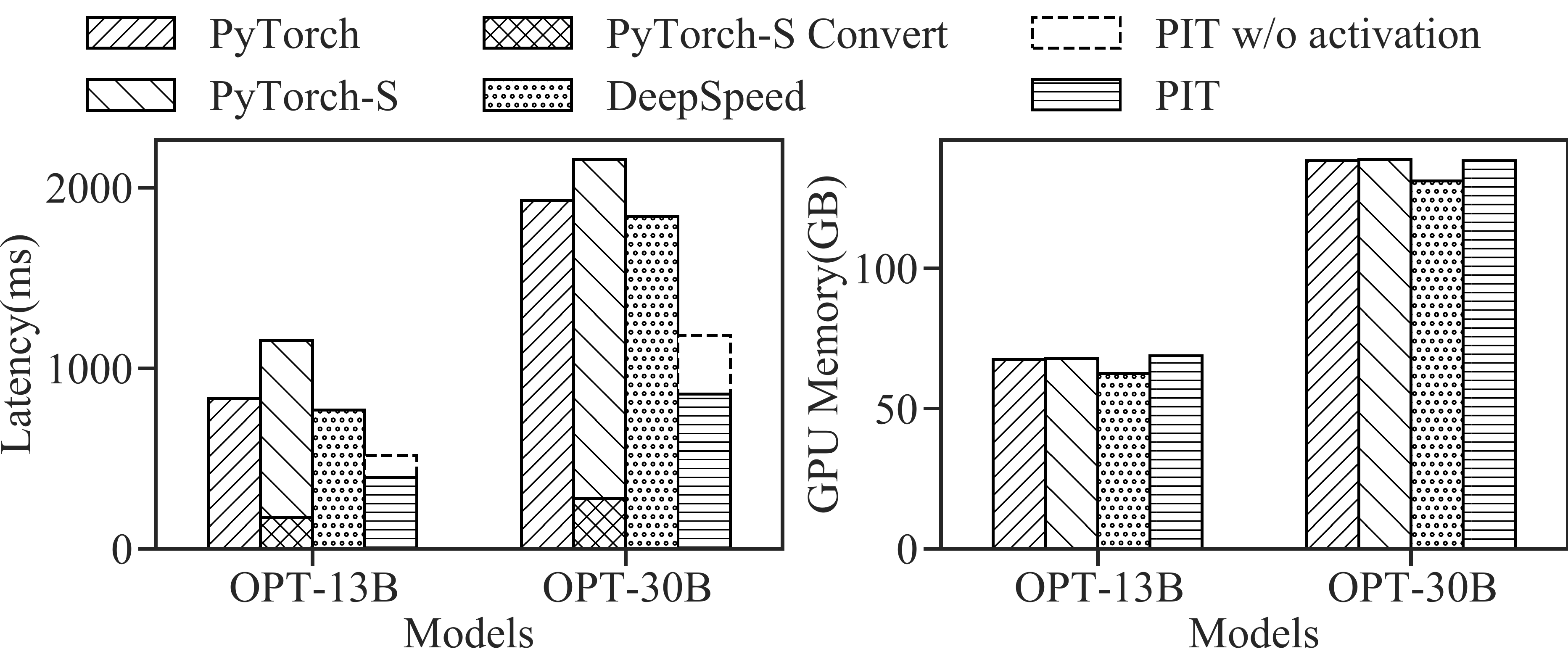}
  % \subfloat[Latency]{
  %   \label{fig:exp_opt_latency}
  %   \includegraphics[width=0.49\columnwidth]{Figure/experiments/speedup_comparation_opt.pdf}}
  % \subfloat[Memory]{
  %   \label{fig:exp_opt_mem}
  %   \includegraphics[width=0.49\columnwidth]{Figure/experiments/memory_comparation_opt.pdf}}
\vspace{0.1cm}
  \caption{End-to-end latency per batch and memory footprints of OPT.}
  \vspace{-0.3cm}
  \label{fig:exp_opt}
\end{figure}

\para{OPT}
is a decoder-only large language model~\cite{zhang2022opt}. We evaluated two versions with 13B and 30B parameters with the Alpaca~\cite{alpaca} dataset on eight V100-32GB GPUs. \sysname{} applies two dynamic sparsity optimizations on OPT: (1) eliminating the padding overhead from sentences with varying lengths in the same batch and (2) exploiting the fine-grained sparsity (up to 99\%) created by the ReLU activation in the FFN layer. The batch size is set to 32. PyTorch-S uses Triton as the backend.

\autoref{fig:exp_opt} compares the end-to-end latency and memory footprints of the OPT. \sysname{} outperforms PyTorch, PyTorch-S, and DeepSpeed by 2.1x$\sim$2.3x, 2.5x$\sim$3.0x and 2.0x$\sim$2.2x, respectively. \revise{The benefit of avoiding padding in dynamic sequences helps \sysname{} to achieve 1.6x$\sim$1.7x speedup against the baselines (\ie \sysname{} w/o activation in \autoref{fig:exp_opt}). By further exploiting the dynamic sparsity in the ReLU activation of FFN layers, \sysname{} further boosts the performance by 1.3x$\sim$1.4x.} 
% \sysname{} outperforms PyTorch and DeepSpeed by leveraging dynamic sparsity within the model for acceleration. 
In contrast to PyTorch-S, which uses Triton block sparse kernel of block size 32x32, \sysname{} performs efficient computations using smaller micro-tile (\ie 1x32) with \sload{} and \swrite{}, thus avoiding computation waste. \revise{Also, PyTorch-S suffers from the sparse format conversion overhead and thus has the highest latency.} In terms of memory consumption, DeepSpeed has the lowest memory usage as it fuses the entire encoder layer into one operator and saves activation memory. \sysname{} has a memory footprint similar to the other baselines. % ``PIT w/o activation'' shows the latency when \sysname{} only optimizes the sparsity due to paddings. 

%\subsubsection{BERT}\label{sec:eval_seqlen}
\begin{comment}    
\begin{figure}[htbp]
    \centering
    \includegraphics[width=0.98\columnwidth]{Figure/experiments/datasets_distribution_seq_len.pdf}
    \caption{\label{fig:exp_seqlen_distribution}The seq length distribution of difference datasets.}
\end{figure}
\end{comment}
\begin{figure*}[t]
  \centering
  % \captionsetup[subfigure]{labelformat=empty}
  \includegraphics[width=2.1\columnwidth]{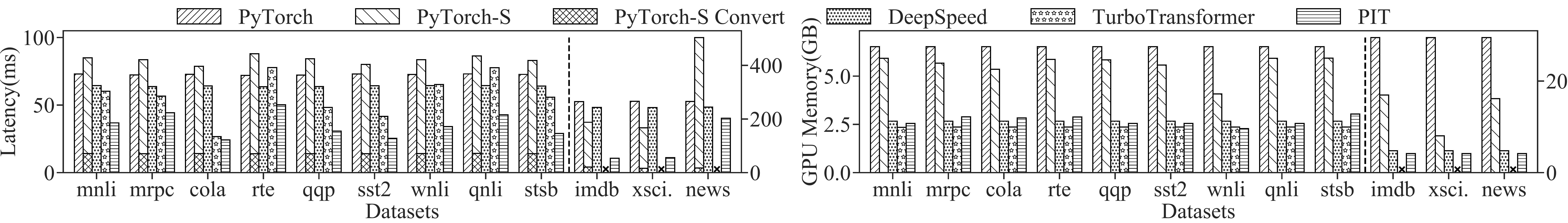}
  % \subfloat[Latency]{
    % \label{fig:exp_seq_lat}
  %   \includegraphics[width=0.98\columnwidth]{Figure/experiments/speedup_comparation_seq_len_glue.pdf}} 
  %   % \hspace{4em}
  % \subfloat[Memory]{
  %   \label{fig:exp_seq_mem}
  %   \includegraphics[width=0.98\columnwidth]{Figure/experiments/memory_comparation_seq_len_glue.pdf}} \\
\vspace{0.1cm}
  \caption{End-to-end latency and memory footprints of BERT on V100.}
  \label{fig:exp_dynamic_seq}
\end{figure*}

%In both training and inference, the batch mechanism is commonly used to improve accelerator utilization. However, batching sequences with different lengths requires padding other sentences to the maximum length, which results in computation waste, as shown in \autoref{fig:dynamic_sparse_case_seq_len}. This padding can be treated as dynamic sparsity that varies with the effective sequence lengths.

%We evaluate \sysname{} on the BERT model from Huggingface Library\cite{wolf2020transformers} and different datasets including GLUE\cite{wang2018glue} (with mnli, mrpc, cola, rtc, qqp, sst2, wnli, qnli, stsb) and IMDB\cite{maas2011learning}, Multi-XScience\cite{lu2020multi}, Multi-News\cite{fabbri2019multi}. \autoref{fig:exp_dynamic_seq} shows the latency and memory footprints of different approaches on these datasets. The batch size is 32. PyTorch-S uses Triton as the backend sparse library.
\para{BERT.}
We also tested \sysname{} on BERT-base~\cite{bert}, an encoder-only language model, using Float32 on a single NVIDIA V100-32GB.
In BERT, \sysname{} only optimizes the dynamic sparsity caused by the different sequence lengths in the same batch. To evaluate its performance more comprehensively, we experimented with various datasets, including GLUE~\cite{wang2018glue} (with mnli, mrpc, cola, rtc, qqp, sst2, wnli, qnli, stsb), IMDB~\cite{maas2011learning}, Multi-XScience~\cite{lu2020multi}, and Multi-News~\cite{fabbri2019multi}. We add TurboTransformer~\cite{turbo} as a baseline, an inference framework that is specifically optimized for variable input sequence lengths by using smart dynamic batching. We used a batch size of 32. PyTorch-S employed Triton as the sparse backend library. %\textbf{ \autoref{fig:exp_dynamic_seq} shows the latency and memory footprints of different approaches on these datasets. }

As \autoref{fig:exp_dynamic_seq} illustrates, \sysname{} outperforms PyTorch, PyTorch-S, DeepSpeed, and TurboTransformer by a factor of 1.3x to 4.9x, 1.8x$\sim$3.5x, 1.2x$\sim$4.5x, and 1.1x$\sim$1.9x, respectively. %Sparse computation with large sparsity granularity causes computation waste, especially for short sequences. This accounts for the poor performance of PyTorch-S on the GLUE dataset, which has a sequence length of 128, compared to its better performance on long document datasets (\eg imdb, xsci., news).
\revise{PyTorch-S performs poorly when the sequence lengths are short (\eg no more than 128 tokens in GLUE). Because its backend (Triton) requires a coarse-grained sparsity (\ie 32 tokens in Triton's kernel), PyTorch-S needs to pad the input sequence to multiple of 32. This incurs a high waste when the sequence lengths are short (\eg a sequence of 16 tokens has to be padded to 32 causing a waste of 50\%).
%In contrast, \sysname{} is more efficient by using a fine-grained sparsity granularity (i.e., 1 token). 
TurboTransformer outperforms the other baselines by dividing the input into multiple small batches based on sentence lengths and processing them sequentially to avoid waste. \sysname{} further outperforms TurboTransformer by processing the whole batch in parallel without waste.} On memory usage, \sysname{} consumes less memory than PyTorch and PyTorch-S, and similar memory to DeepSpeed and TurboTransformer. DeepSpeed and TurboTransformer optimize memory usage by fusing the entire layer into a single operator to reduce activation memory, \revise{which is compatible with \sysname{}'s sparsity optimizations. } TurboTransformer crashes when the input sequence length increases due to kernel implementation issues.
%``PyTorch-S Convert'' means the sparse index construction overhead of PyTorch-S, which accounts for 14.0\%$\sim$20.9\% in the end-to-end latency. Thanks to the fast index construction (see \S~\ref{sec:fast_index}), the sparse indices construction overhead of \sysname{} is negligible.
%\autoref{fig:exp_dynamic_seq} shows the memory footprints of \sysname{} under different datasets. Both PyTorch-S and \sysname{} use less memory than PyTorch due to less padding. %Although PyTorch-S has a higher latency than PyTorch on datasets with short sequences, PyTorch-S still has a smaller memory footprint. 
%\sysname{} further reduces 44.0\%$\sim$75.5\% GPU memory usage than PyTorch-S, because \stile{} can be more efficient on small sparse granularity.

\begin{figure}[t]
  \centering
  % \captionsetup[subfigure]{labelformat=empty}
  \includegraphics[width=1\columnwidth]{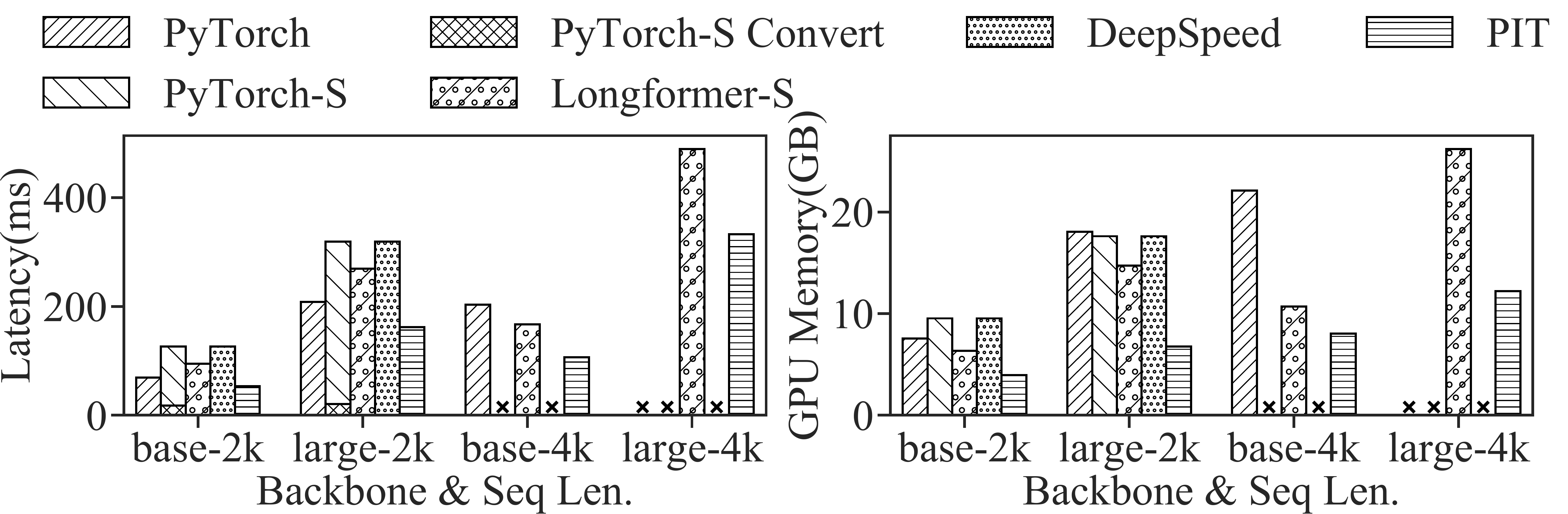}
  % \subfloat[Latency]{
  %   \label{fig:exp_longformer_lat}
  %   \includegraphics[width=0.53\columnwidth]{Figure/experiments/speedup_comparation_longformer.pdf}} 
  % \subfloat[Memory]{
  %   \label{fig:exp_longformer_mem}
  %   \includegraphics[width=0.50\columnwidth]{Figure/experiments/memory_comparation_longformer.pdf}} \\
\vspace{0.1cm}
  \caption{End-to-end inference latency and memory footprint of Longformer on V100.}
  \label{fig:exp_longformer}
\end{figure}

\para{Longformer}
% Inspired by the biological neural network, dynamic sparse models can adaptively apply different sparsity patterns for different inputs. Longformer\cite{longformer} and MuseFormer\cite{museformer} are two representative dynamic sparse models on NLP and speech tasks respectively. We evaluate \sysname{} under the dynamic sparsity pattern of both Longformer and Museformer.
is an encoder-only language model with dynamic sparse attention~\cite{longformer}. Longformer adaptively pays attention to several important words (\eg class token) of the input. The position of dynamic attention varies for different inputs, \revise{which is the source of dynamic sparsity.} \autoref{fig:exp_longformer} illustrates Longformer's inference latency and memory cost for input sequence lengths of 2048 and 4096. 
%We use the Longformer base and large model from official homepage\cite{longformer2020code} and long documents dataset Arxiv\cite{cohan2018discourse}.
\sysname{} optimizes the dynamic sparsity in the dynamic sparse attention.
To evaluate comprehensively, we also add the sparse implementation specifically optimized for the Longformer (represented by Longformer-S~\cite{longformer2020code}). PyTorch-S selects Triton as the backend. 
% Besides, the Longformer team has implemented a special optimized implementation specifically for the sparse pattern of  and PyTorch-S selected Triton as its backend sparse libraries.
\autoref{fig:exp_longformer} shows the latency and GPU memory usage on a single V100. \sysname{} is faster than PyTorch, Longformer-S, PyTorch-S, and DeepSpeed by up to 1.9x, 1.8x, 2.4x, and 2.4x, respectively. Longformer-S outperforms PyTorch-S because of its specifically optimized GPU kernels for its designed sparsity pattern through sparse pattern decomposition. However, its design is hard to be used by other models. DeepSpeed uses Triton to implement their sparse attention, so it has a similar performance to PyTorch-S.
\revise{The index construction overhead \revise{of PyTorch-S} is shown in ``PyTorch-S Convert'', which accounts for 6.3\% $\sim$13.9\% of the end-to-end latency. } 
% The performance gain of \sysname{} over Longformer-S and PyTorch-S comes mainly from the \sload{} and \swrite{}.
Moreover, PyTorch-S performs even worse when it selects the fine-grained sparse library as the backend because the sparsity ratio is not high enough. When PyTorch-S selects Triton (Block Sparse) as the backend, it is slower than \sysname{} due to the wasted computation caused by the dynamic global attention\cite{longformer}.
\revise{Longformer-S is more efficient since it has no computation waste by rearranging the input tensor}, but it introduces large data rearrangement overheads. In contrast, \sysname{} organizes the small micro-tiles on the fly with negligible overheads and computes them in an efficient dense computation tile directly without computation waste using \sload{} and \swrite{}.
As for memory usage, \sysname{} uses the least memory. PyTorch-S \revise{and DeepSpeed} crashed due to out-of-memory when the input sequence length reached 4096. Both PyTorch-S and DeepSpeed have to use block sparse ($32\times 32$ in Triton) to cover all remaining values, leading to computation waste and a higher sparsity ratio. Longformer-S introduces extra memory cost due to its data re-arrangement, which creates many temporary intermediate tensors.

\begin{figure}[htbp]
  \centering
  % \captionsetup[subfigure]{labelformat=empty}
\includegraphics[width=1\columnwidth]{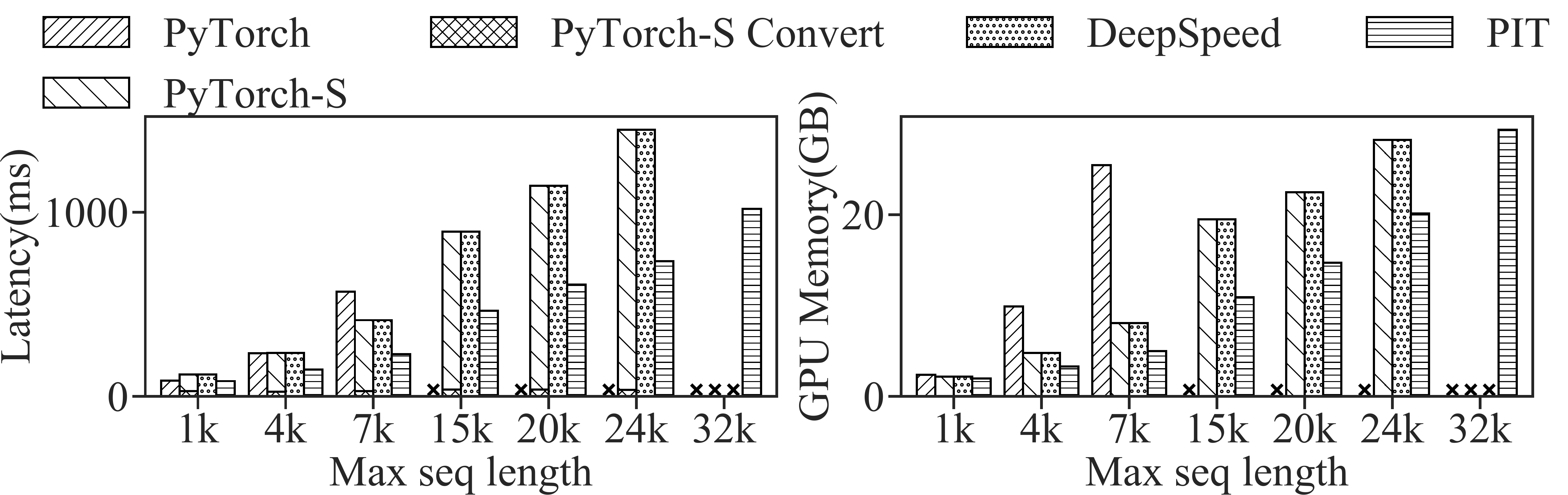}
  % \subfloat[Latency]{
  %   \label{fig:exp_museformer_lat}
  %   \includegraphics[width=0.86\columnwidth]{Figure/experiments/speedup_comparation_museformer.pdf}} \\
  % \subfloat[Memory]{
  %   \label{fig:exp_museformer_mem}
  %   \includegraphics[width=0.8\columnwidth]{Figure/experiments/memory_comparation_museformer.pdf}} 
  \caption{End-to-end inference latency and memory footprint of Museformer on V100.}
  \label{fig:exp_museformer}
\end{figure}

\para{Museformer} is a decoder-only language model that also generates the dynamic sparsity pattern according to the input data to improve model performance~\cite{museformer}. \autoref{fig:exp_museformer} shows the inference latency and memory footprint under different input sequence lengths. \sysname{} is 2.5x, 2.0x, and 2.0x faster than PyTorch, PyTorch-S, and DeepSpeed respectively before they crash due to out-of-memory. The time of sparse index construction accounts for up to 23.2\% of the end-to-end latency in PyTorch-S for short sequences. As the input length increases, the amount of calculation becomes larger, and the time proportion of index construction can be gradually diluted.
As for memory usage, \sysname{} shows the lowest memory footprint. PyTorch consumes much more memory because it cannot understand and optimize the dynamic sparsity. Compared to PyTorch-S and DeepSpeed, \sysname{} reduces computation waste by PIT transformation, resulting in lower memory consumption.

\begin{comment}
\begin{table}[t]
    \centering
    \setlength{\tabcolsep}{1mm}
    \resizebox{\columnwidth}{!}{
    \begin{tabular}{lccccc}
    \toprule
        \textbf{Backbone} & \textbf{pytorch} & \textbf{pytorch sparse} & \textbf{pytorch + SoTA kernel} & \textbf{sparta} & \textbf{Speedup Ratio} \\
       \midrule
        longformer-base-2048 & 47.93 & 63.07 & 43.97 & 33.74 & 1.42 \\
        longformer-large-2048 & 145.13 & 172.43 & 110.90 & 103.74 & 1.40 \\
        longformer-base-4096 & 142.78 & 106.22 & - & 67.22 & 1.74 \\
        longformer-large-4096 & 408.66 & 306.10 & - & 211.57 & 1.93 \\
        \bottomrule
    \end{tabular}
    }
    \caption{End-to-end training/inference cost comparison in Longformer scenarios.}
    \label{tab:speedup_comparison_e2e_longformer}
\end{table}
\end{comment}
\subsection{End-to-End training}\label{sec:eval_end_2_end_train}
In this section, we evaluate \sysname{} on both NVIDIA A100 and V100 GPUs to demonstrate its superior performance when using dynamic sparsity to accelerate training. 

\para{OPT Training.}
\begin{figure}[t]
  \centering
  % \captionsetup[subfigure]{labelformat=empty}
  \includegraphics[width=1.0\columnwidth]{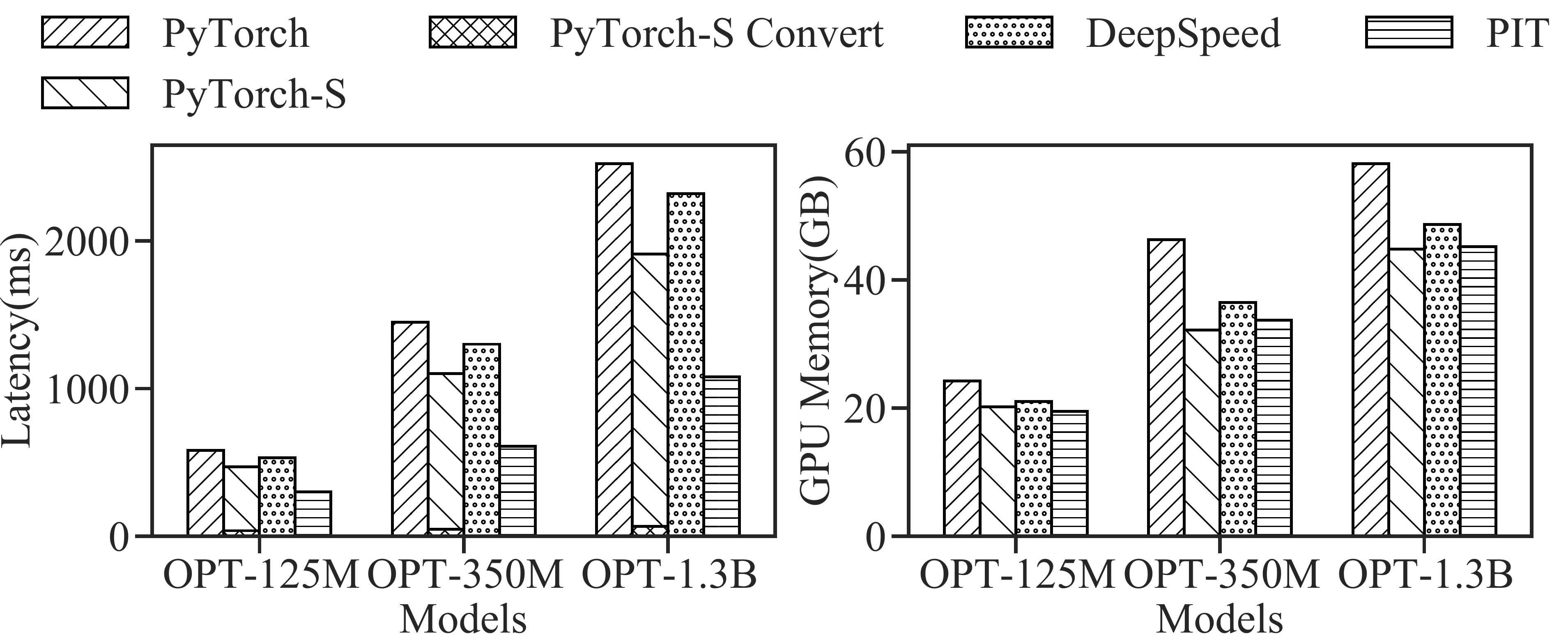}
  % \subfloat[Latency]{
  %   \label{fig:exp_opt_training_latency}
  %   \includegraphics[width=0.49\columnwidth]{Figure/experiments/speedup_comparation_opt_training.pdf}}
  % \subfloat[Memory]{
  %   \label{fig:exp_opt_training_mem}
  %   \includegraphics[width=0.49\columnwidth]{Figure/experiments/memory_comparation_opt_training.pdf}}
  \caption{End-to-end latency per batch and memory footprints of OPT training.}
  \label{fig:exp_opt_training}
\end{figure}
We fine-tuned the 125M, 350M, and 1.3B OPT models using the Alpaca dataset on an NVIDIA A100-80GB GPU.  We utilized \sysname{} to optimize the dynamic sparsity due to varying sentence lengths within the same batch. Due to memory limitations, we set the training batch size to 8. In our experiment, we compared the performance of \sysname{} with PyTorch, PyTorch-S, and DeepSpeed. The results, as shown in \autoref{fig:exp_opt_training}, demonstrate the time-cost and memory footprint of a forward and backward pass. \sysname{} achieved 1.9x$\sim$2.4x, 1.6x$\sim$1.8x, and 1.8x$\sim$2.2x faster speed than PyTorch, PyTorch-S, and DeepSpeed, respectively. \revise{Similar to the inference optimization of varying sentence lengths, \sysname{} saves the computation of padding in PyTorch and DeepSpeed. Compared to PyTorch-S, \sysname{} supports more fine-grained sparsity granularity (1 token) than Triton's block sparse granularity (32 tokens) leading to more efficient computation.} Also, \sysname{} saves memory access overheads caused by reformatting data from dense to sparse formats in PyTorch-S.
In terms of memory footprint, \sysname{} and PyTorch-S have the smallest memory footprints during training with dynamic sparsity. \revise{Compared to inference, DeepSpeed cannot save the activation memory by fusing the entire layer into one operator in training, thus leading to more memory consumption.}

%\subsubsection{Sparse Training} \label{sec:eval_pruning}
\begin{figure*}[t]
  \centering
  % \captionsetup[subfigure]{labelformat=empty}
  \subfloat[Latency]{
    \label{fig:exp_nnpruning_lat}
    \includegraphics[width=0.99\columnwidth]{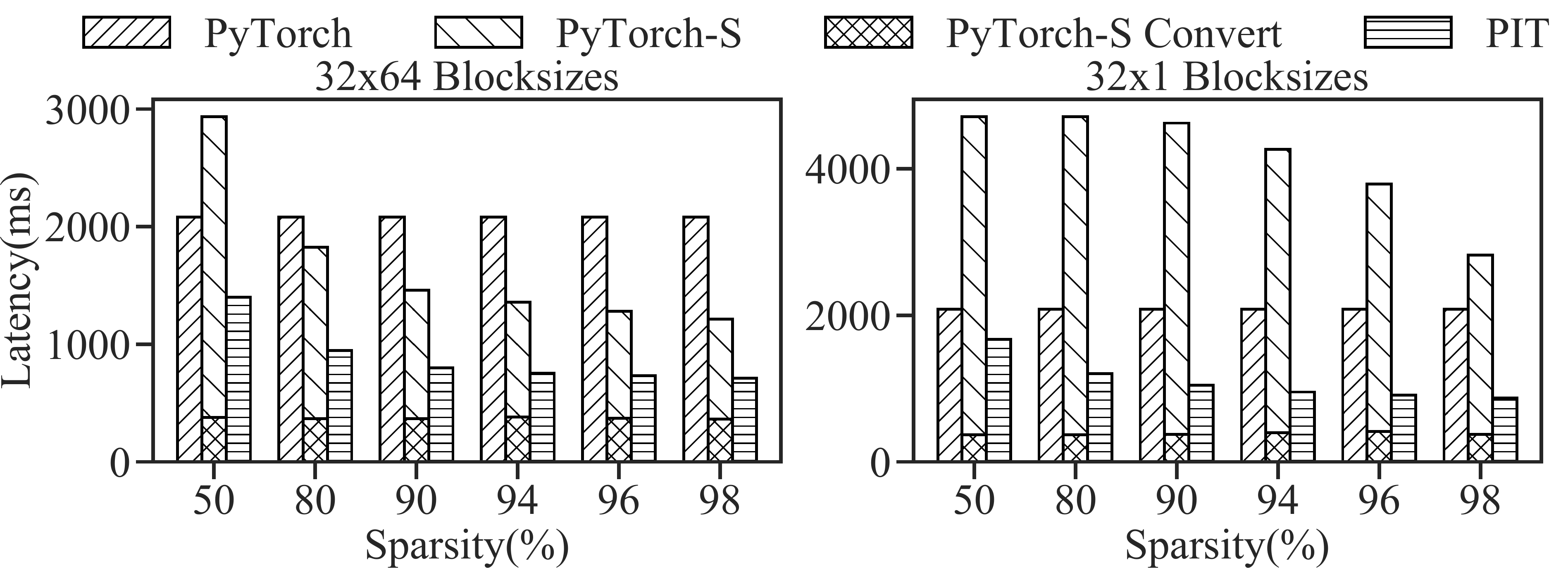}} 
  \subfloat[Memory]{
    \label{fig:exp_nnpruning_mem}
    \includegraphics[width=0.99\columnwidth]{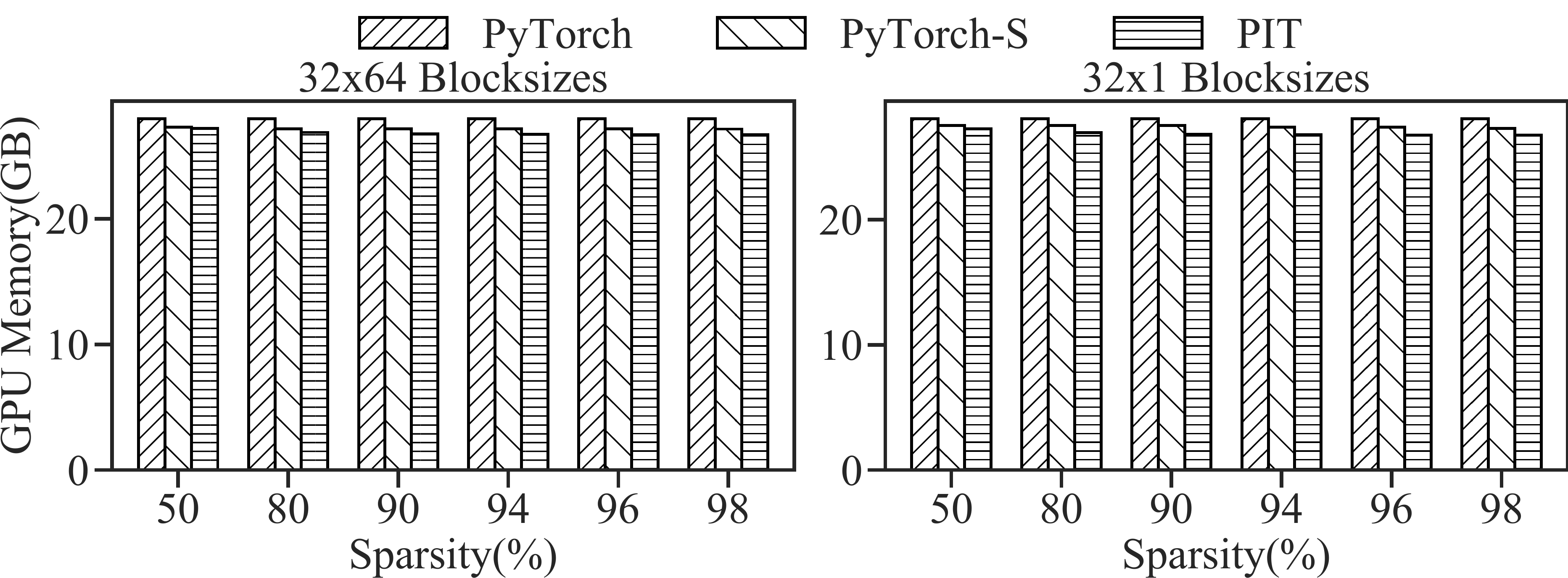}} \\
  % \subfloat[Latency of 32x1 Blocksize]{
  %   \label{fig:exp_nnpruning_32_1_lat}
  %   \includegraphics[width=0.95\columnwidth]{Figure/experiments/speedup_comparation_nn_pruning_32x1.pdf}} 
  % \subfloat[Memory of 32x1 Blocksize]{
  %   \label{fig:exp_nnpruning_32_1_mem}
  %   \includegraphics[width=0.95\columnwidth]{Figure/experiments/memory_comparation_nn_pruning_32x1.pdf}} \\
  \caption{End-to-end latency per batch and memory footprints of magnitude iterative pruning under $32\times 64$ and $32\times 1$ block.}
  \label{fig:exp_nnpruning_32_64}
\end{figure*}

\para{Sparse Training.}
We use iterative pruning \cite{li2020train}, a common approach of model compression, to demonstrate how \sysname{} can leverage dynamic sparsity to speed up sparse training. At each step, the pruning algorithm generates a mask based on the weight's magnitude, which varies for different inputs. We evaluated our approach on the GLUE dataset \cite{wang2018glue} using the BERT model, which we pruned using block-wise sparsity at two granularities: $32\times 64$ and $32\times 1$.
\autoref{fig:exp_nnpruning_lat} shows the time to process a batch of data (including both forward and backward passes) for each sparsity granularity setting, with a batch size of 32. PyTorch-S used Triton as the backend. When the sparsity granularity was $32\times 64$, \sysname{} achieved a 1.5x$\sim$3.0x and 1.7x$\sim$2.2x speedup over PyTorch and PyTorch-S, respectively. \revise{Although the sparsity granularity of $32\times 64$ is larger than the computation tile of PyTorch-S's block sparse kernel (i.e., $32\times32$), the sparsity pattern of each layer constantly changed during pruning. It requires every layer of PyTorch-S to rebuild the sparse indices once per batch. This makes PyTorch-S suffer from heavy index construction overhead. \sysname{} outperforms PyTorch-S on the granularity of $32\times 64$ mainly due to its fast index construction.}

\revise{Existing research has already shown a smaller sparsity granularity could bring higher accuracy but challenges execution performance optimization \cite{nn_pruning}.} In our evaluation, using the granularity is $32\times1$, the pruned model %using $32\times1$ 
achieved 0.26\%$\sim$ 0.72\% accuracy gain on the MNLI dataset and 0.11\% $\sim$ 1.26\% accuracy gain on the SST-2 dataset. For latency, \sysname{} outperformed PyTorch and PyTorch-S by 2.4x and 4.8x, respectively. \revise{The computation time of PyTorch-S increased significantly due to the misaligned data sparsity granularity ($32\times1$) and the block-sparse GPU kernels ($32\times32$ or $16\times16$). It is worth noting that \sysname{} achieved almost the same speed as $32\times64$ when the sparsity granularity was set to $32\times1$, because \sysname{} can use the fine-grained $32\times 1$ micro-tile to cover the sparse data while using the $32\times64$ kernel for the most efficient computation, achieving the best of both worlds.}

\autoref{fig:exp_nnpruning_mem} shows the memory usage during the pruning process. \sysname{} used the least memory compared to PyTorch and PyTorch-S. The memory footprint only dropped slightly as the sparsity ratio increased, as the iterative pruning algorithm only prunes the model weights during the pruning, and weight tensors take up only a small fraction of memory. In contrast, PyTorch and PyTorch-S store the dense weights and gradients, so their memory footprint almost does not change with the sparsity ratio.

\subsection{Effectiveness of PIT transformation}
\label{sec:eval_effect_stile}
\begin{comment}
    
\begin{figure*}
  \centering
  % \captionsetup[subfigure]{labelformat=empty}
  \subfloat[]{
    \label{fig:exp_stile_4096}
    \includegraphics[width=2\columnwidth]{Figure/experiments/micro_benermark_4096_stile.pdf}} \\
  \subfloat[]{
    \label{fig:exp_stile_4096_small}
    \includegraphics[width=2\columnwidth]{Figure/experiments/micro_benermark_4096_stile_small.pdf}}
  \caption{Comparison of cuSPARSE, Sputnik, OpenAI Block Sparse, SparTA and \sysname{} on matrix multiplication($4096\times4096\times4096$) under different sparsity ratios and granularities.}
  \label{fig:exp_stile_microbench}
\end{figure*}
\end{comment}
% \begin{comment}
\begin{figure}
    \centering
    \includegraphics[width=1\columnwidth]{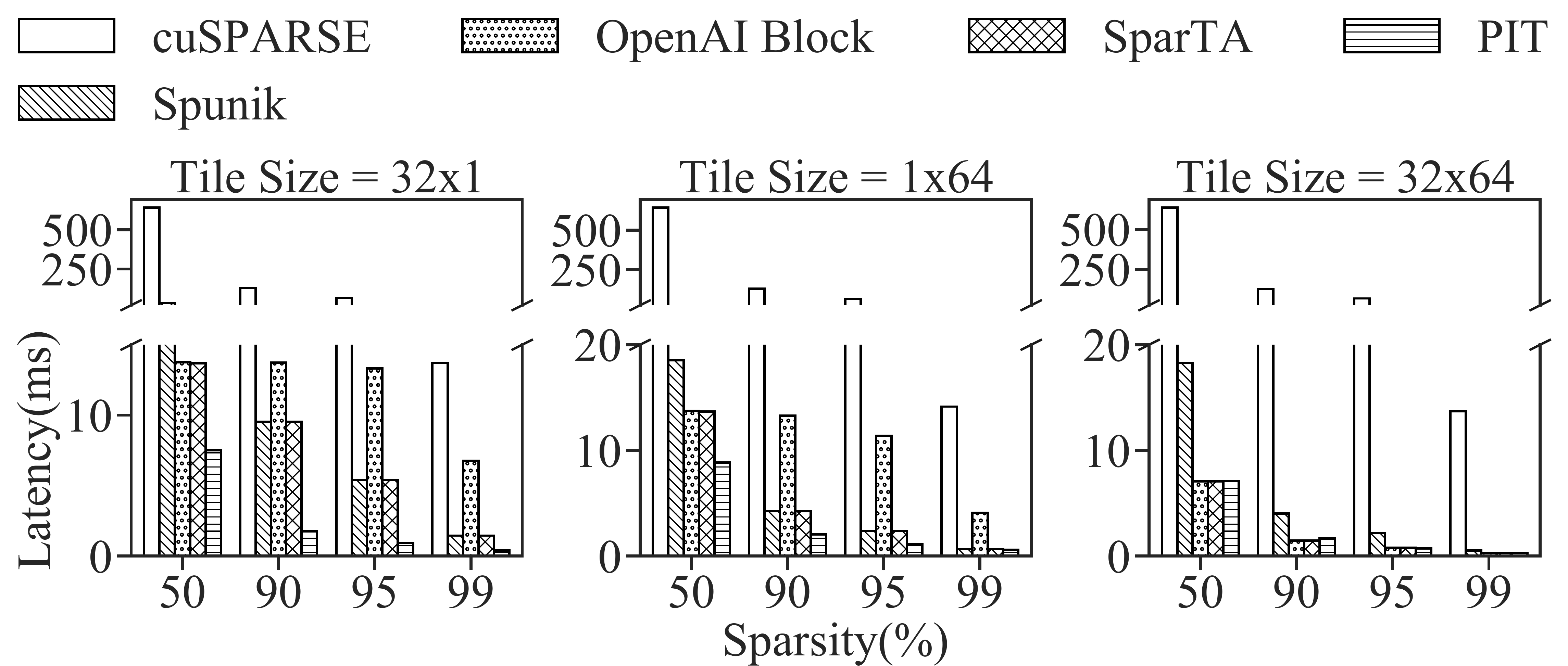}
    \caption{\label{fig:exp_stile_4096}Comparison of cuSPARSE, Sputnik, OpenAI Block Sparse, SparTA and \sysname{} on matrix multiplication ($4096\times 4096\times 4096$) under different sparsity ratios.}
\end{figure}
% \end{comment}

\noindent\textbf{PIT transformation on dense kernels.}
In this study, we evaluate the performance of \sysname{}'s sparse matrix multiplication kernel with different sparsity granularities and shapes. To demonstrate the effectiveness of PIT transformation, we compare \sysname{} with several other sparse libraries, including cuSPARSE, Sputnik, and OpenAI Block Sparse (Triton). We also introduce SparTA~\cite{sparta}, a state-of-the-art sparse deep-learning compiler designed for static sparsity optimization. \revise{For this experiment, we use a static sparsity pattern to evaluate the computation efficiency, therefore the latency results shown in \autoref{fig:exp_stile_4096} do not include conversion or compiling overhead.}
When the sparsity granularity is $32\times64$, \sysname{}, SparTA, and OpenAI Block Sparse have similar latency because they use the same dense computation tile. cuSPARSE and Sputnik perform poorly due to their inefficient fine-grained computation granularity. \revise{When the sparsity granularity is $32\times1$ and $1\times64$, Sputnik and SparTA outperform cuSPARSE and OpenAI Block Sparse due to better granularity alignment. For the sparsity granularity of $32\times1$, \sysname{} is 4.3x$\sim$5.8x faster than Sputnik and 1.5x$\sim$5.7x faster than SparTA. For the sparsity granularity of $64\times1$, \sysname{} is 1.1x$\sim$2.3x faster than Sputnik and 1.1x$\sim$2.2x faster than SparTA. The PIT transformation contributes to the performance gain of \sysname{} over Sputnik and SparTA by allowing \sysname{} 
%to use a more compute-efficient $64\times64$ tile to cover multiple $1x\times 4$ data at once
to perform the efficient computation even under a small sparsity granularity (\ie $32\times 1$)
}. In addition, \sysname{} has a negligible overhead on \sload{} and \swrite{} running at a speed close to the original dense computation tile for different sparsity granularities.

\noindent\textbf{PIT transformation on hardware instructions.}
In this experiment, we illustrate how PIT transformation can loosen the constraints on hardware instructions, like wmma~\cite{wmma}. We conduct sparse matrix multiplication on two different sparsity granularities ($32\times1$ and $32\times64$) for a $[4096,4096]\times[4096,4096]$ matrix multiplication. The first input tensor is sparse and stored in column-major format. The wmma instruction only supports three shapes ($[16,16]\times[16,16]$, $[32,8]\times[8,16]$, $[8,32]\times[32,16]$) in half-precision, making it unsuitable for a $32\times1$ sparsity granularity. We use PIT transformation to derive two sparse kernels with micro-tiles of $32\times1$ and $32\times64$, respectively. We apply the sparse kernel with a micro-tile of $32\times1$ to perform sparse matrix multiplication with $32\times1$ sparsity granularity, while the kernel with a micro-tile of $32\times64$ is used for $32\times64$ sparsity granularity. \autoref{fig:exp_stile_4096} shows the two sparse kernels generated by PIT have similar latency at different sparsity ratios. It proves PIT transformation introduces little overhead.

\begin{figure}
    \centering
    \includegraphics[width=0.6\columnwidth]{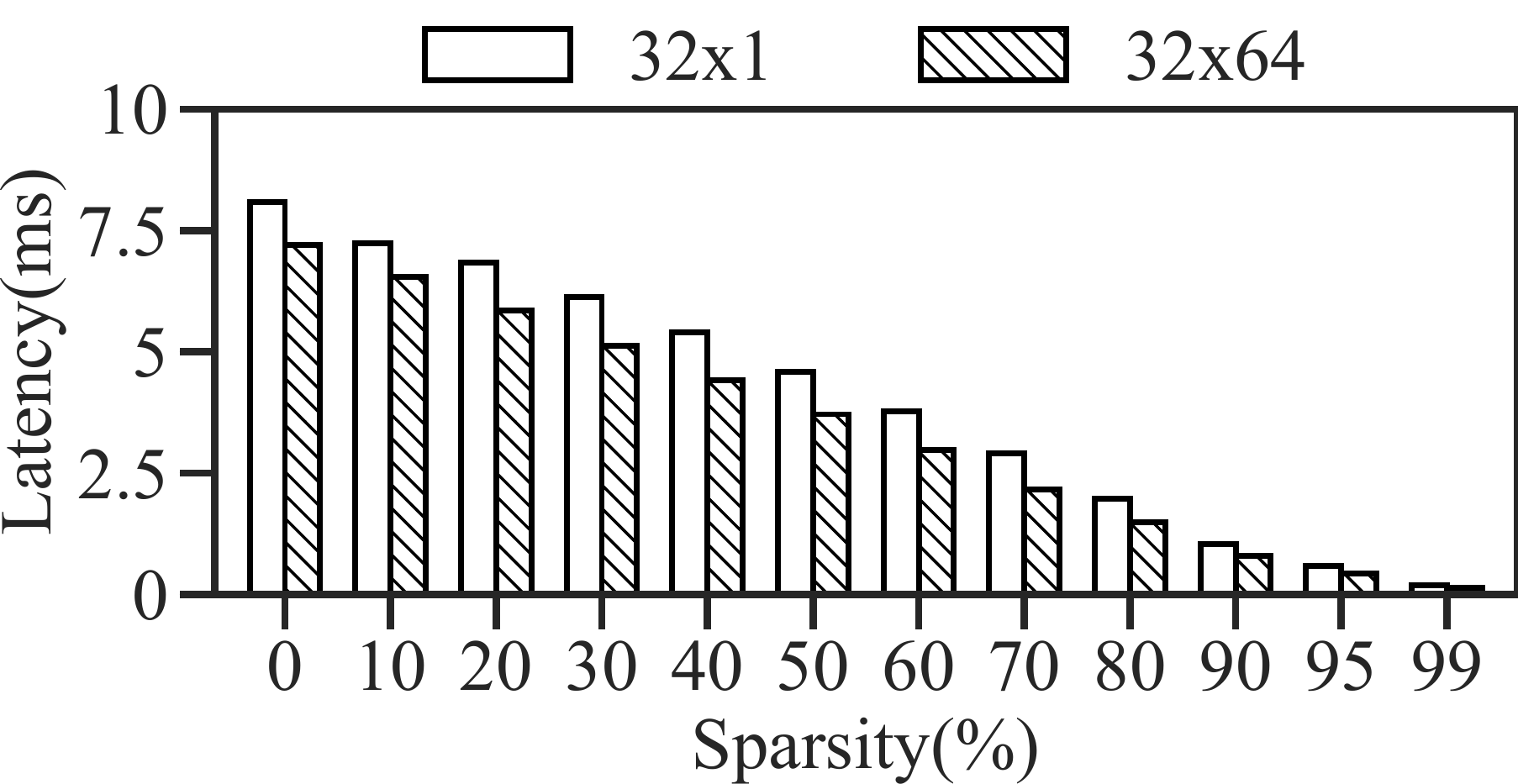}
    \caption{\label{fig:exp_tensorcore} Latency of \sysname{} with Tensor Core.}
\end{figure}
\subsection{Conversion Overhead}

In this experiment, we highlight the conversion overhead of \sysname{} in more detail. Specifically, we evaluate the conversion overhead under different sparsity granularities and sparsity ratios \revise{on a V100-32GB GPU\footnote{Different from V100-32GB, PIT has a latency similar to PyTorch-S on V100-16GB for tile size 1x1. The performance advantage of PIT over other tile size on V100-16GB remains the same as in \autoref{fig:exp_convert_4096}.}}. \autoref{fig:exp_convert_4096} shows the sparse index construction time of \sysname{} and PyTorch-S. PyTorch-S selects the index construction function provided by cuSPARSE when the granularity is $1\times1$ and the function provided by Triton when the granularity is $16\times16$ and $32\times32$. \sysname{} is 3.6x$\sim$4.7x faster than cuSPARSE when the granularity is $1\times1$, 11.2x$\sim$14.2x faster than Triton when the granularity is $16\times16$, and 13.3x$\sim$26.5x faster than Triton's index construction when the granularity is $32\times32$. As far as we know, \sysname{} is the first to support the fast online index construction for all kinds of sparsity granularities.

\label{sec:eval_convert}
\begin{figure}%[b]
  \centering
   \includegraphics[width=0.95\columnwidth]{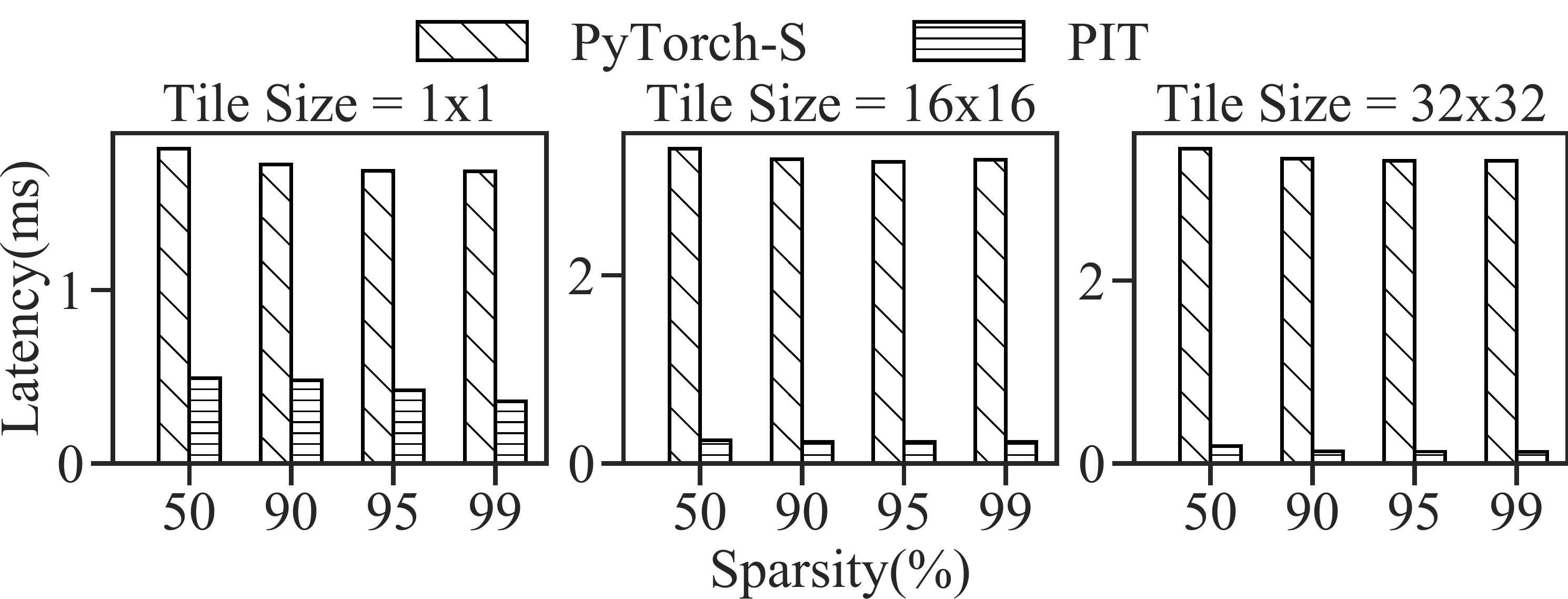}
  \caption{\label{fig:exp_convert_4096}The index construction latency of tensor with shape of $4096\times4096$.}
  %\label{fig:exp_convert}
\end{figure}

\begin{figure}[t]%[h]
  \centering
  % \captionsetup[subfigure]{labelformat=empty}
  \includegraphics[width=1\columnwidth]{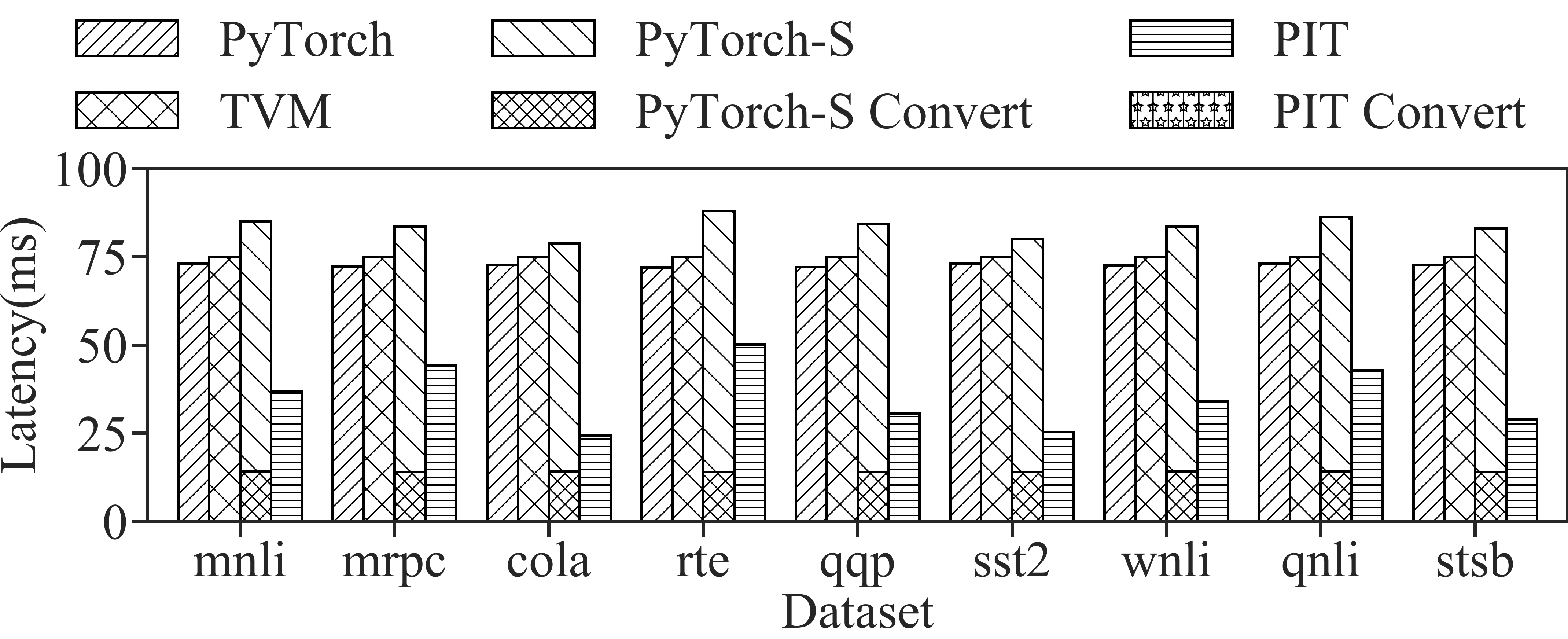}
  % \subfloat[Latency]{
  %   \label{fig:exp_longformer_lat}
  %   \includegraphics[width=0.53\columnwidth]{Figure/experiments/speedup_comparation_longformer.pdf}} 
  % \subfloat[Memory]{
  %   \label{fig:exp_longformer_mem}
  %   \includegraphics[width=0.50\columnwidth]{Figure/experiments/memory_comparation_longformer.pdf}} \\
% \vspace{0.1cm}
  \caption{\revise{End-to-end conversion overhead of \sysname{}.} }
  \label{fig:exp_tvm}
\end{figure}

\revise{To further evaluate the conversion overhead, we also measure the proportion of the conversion overhead in the end-to-end latency of \sysname{}. Specifically, PIT optimizes the dynamic sparsity caused by different sequence lengths in the BERT model. In addition to PyTorch, we also introduce TVM, a popular dense tensor compiler. Each task in TVM is finetuned 2000 steps by Auto-Scheduler (Ansor~\cite{zheng2020ansor}). \autoref{fig:exp_tvm} shows the end-to-end latency of \sysname{} and the baselines, including the index construction overhead of \sysname{} (``\sysname{} Convert'') and PyTorch-S (``PyTorch-S Convert''). The conversion overhead of \sysname{} accounts for 0.7\% to 1.1\% of the end-to-end latency, which is almost invisible in the figure.}

\subsection{Micro-Tile Online Searching} \label{sec:eval_online_search}
Different sparsity patterns and different sparsity ratios may lead to different optimal micro-tiles. \sysname{} considers the efficiency of the computation kernel and the computation waste at the same time to find the best micro-tile configuration. Table~\ref{tab:stile_online_search} shows the searched micro-tiles of $4096\times4096\times4096$ matrix multiplication under different sparsity granularities and ratios. Take the first line of Table~\ref{tab:stile_online_search} as an example. The algorithm finds it most efficient to use a micro-tile of $16\times1$ to cover the granularity of $2\times1$ when the sparsity ratio is 95\%. \revise{The micro-tile of $16\times 1$ for $2\times 1$ data leads to a sparsity ratio of 66.39\% in \sysname{}'s computation.} %The granularity of the sparsity pattern given by the algorithm is $2\times1$, and the sparsity ratio is 95\%. 
%\sysname{} will use the data tile with the shape of $16\times1$ to cover the remained values. 
%After covering,  
The micro-tile $16\times1$ is derived from the dense computation tile $16\times32\times$128 by applying PIT transformation on the second axis of the first input tensor. \sysname{} balances the trade-off between the efficiency of the kernel and the computation waste on the fly. It takes 30us$\sim$100us for \sysname{} to search for the best micro-tile and the corresponding dense computation tile, which is fast enough for online searching. %Therefore, 

\begin{comment}
\begin{figure*}
    \centering
    \includegraphics[width=1.98\columnwidth]{Figure/experiments/micro_benermark_4096_stile_small.pdf}
    \caption{\label{fig:exp_stile_4096}Comparison of cuSPARSE, Sputnik, OpenAI Block Sparse, SparTA and \sysname{} on matrix multiplication ($4096\times 4096\times 4096$) under different sparsity ratios.}
\end{figure*}
\end{comment}

\begin{table}[t]
    \centering
    \footnotesize
    \setlength{\tabcolsep}{1mm}
    \resizebox{\columnwidth}{!}{
    \begin{tabular}{l|c|c|c|c|c}
    \toprule
        \makecell[c]{Sparsity\\Granularity} & \makecell[c]{Origin \\Sparsity \\Ratio(\%)} & Micro Tile & \makecell[c]{Sparsity \\Ratio After \\ Cover (\%)} & \makecell[c]{Origin \\Dense Kernel} & \makecell[c]{Latency\\(ms)} \\
        \midrule\midrule
        (2,1) & 95 & (16, 1) & 66.39 & $[16,32]\times[32,128]$ & 8.04 \\
        (2,1) & 99 & (8, 1) & 96.06 & $[8,32]\times[32,128]$ & 2.34 \\
        (4,1) & 95 & (16, 1) & 81.45 & $[16,32]\times[32,128]$ & 4.29 \\
        (4,1) & 99 & (16, 1) & 96.05 & $[16,32]\times[32,128]$ & 1.37 \\
        (8,1) & 95 & (8, 1) & 95 & $[8,32]\times[32,128]$ & 2.34 \\
        (8,1) & 99 & (32, 1) & 96.02 & $[32,64]\times[64,32]$ & 0.90 \\
        (32, 1) & 95 & (32, 1) & 95 & $[32,64]\times[64,32]$ & 0.94 \\
        (32, 1) & 99 & (32, 1) & 99 & $[32,64]\times[64,32]$ & 0.39 \\
        \bottomrule
    \end{tabular}
    }
    \caption{The micro tile online search results for different sparsity granularity and sparsity ratios.}
    \label{tab:stile_online_search}
\end{table}

\begin{figure}%[h]
  \centering
  \includegraphics[width=1\columnwidth]{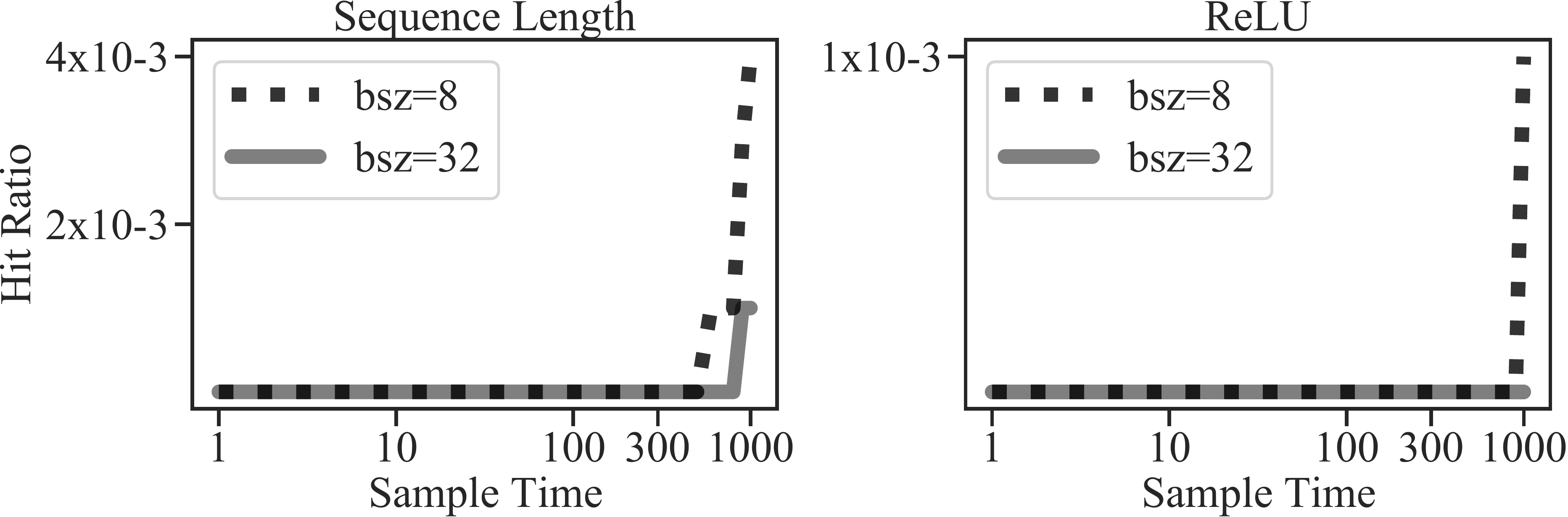}
  \caption{\revise{Investigation of the dynamics of sparsity patterns in varying sequence lengths and ReLU.} }
  \label{fig:exp_hint}
\end{figure}

\subsection{\revise{Dynamic Sparsity Pattern Study}} \label{sec:eval_dynamic_sparse_pattern}
\revise{A potential alternative solution for dynamic sparsity is to memorize frequent sparsity patterns, compile optimized kernels for them, and reuse these kernels when the pattern appears again. In this micro-benchmark, we invalidate this solution by investigating how frequently the dynamic sparse pattern will appear multiple times. We use two widely-existing dynamic sparsity patterns, \ie the varying input sequence lengths and the dynamic sparsity caused by the ReLU operator. Specifically, we traverse the MNLI dataset with different batch sizes (8 and 32) and check whether the sparse pattern of the current input batch has appeared in the previous batches. \autoref{fig:exp_hint} shows the cumulative hit ratio of dynamic sparsity that has appeared in previous batches. We find the repetition ratios of the two sparsity patterns are both notably low. The varying sequence lengths only have 0.4\% requests hitting a sparsity pattern that appeared in the previous batches. The ratio is even lower to 0.1\% for ReLU. Given the dynamic nature of such sparsity, many previous research works for static sparsity optimization are no longer applicable to dynamic sparsity. The sparse kernel optimized for a specific sparsity pattern is almost non-reusable.}
%\subsection{Offline \stile{} Constuction}

%% file: Figure/experiments/inference_models.tex
\begin{table}[t]
    \small
    \centering
    % \footnotesize
    \setlength{\tabcolsep}{1mm}
    % \resizebox{\columnwidth}{
    \begin{tabular}{l|c|c|c|c}
    \toprule
    Models & Datasets & \makecell[c]{Model\\Structure} & Precision & Devices \\
    \midrule
    % \hline
    \hline
    \makecell[l]{Switch\\Transformers\cite{fedus2021switch}} & MNLI~\cite{wang2018glue} & \makecell[c]{Encoder\\Decoder\\ MoE}  & fp16,fp32 & A100 \\
    % \midrule
    \hline
    \revise{Swin-MoE~\cite{hwang2022tutel}} & \revise{ImageNet} & \makecell[c]{\revise{Encoder}\\ \revise{MoE}} & \revise{fp16} & \revise{A100} \\
    \hline
    OPT~\cite{zhang2022opt} & Alpaca~\cite{alpaca} & Decoder & fp32 & V100 \\
    \hline
    BERT~\cite{bert} & \makecell[c]{GLUE~\cite{wang2018glue}, \\News~\cite{fabbri2019multi} etc.} & Encoder & fp32 & V100 \\
    \hline
    Longformer~\cite{longformer} & Arxiv~\cite{cohan2018discourse} & Encoder & fp32 & V100 \\
    \hline
    MuseFormer~\cite{museformer} & LMD~\cite{Raffel2016LearningBasedMF} & Decoder & fp32 & V100 \\
    \bottomrule
    \end{tabular}
    % }
    \caption{Models and datasets in the evaluation. }
    \label{tab:inference_models}
\end{table}

%% file: Relatedwork.tex
\section{Related Works}

\revise{Dynamic sparsity has emerged as a critical area for improving the efficiency of deep learning models. We discuss the different categories of existing solutions, including the software solutions for both static sparsity patterns and dynamic sparsity patterns, as well as the hardware solutions. 
The comparison between \sysname{} and these solutions can help to better understand the technical contribution of \sysname{}.}

\noindent\revise{\textbf{Optimizations for Static Sparsity Patterns.} 
Ahead-of-time compiler-based techniques like SparTA~\cite{sparta}, TACO~\cite{taco},
SparseTIR~\cite{sparsetir}, Tiramisu~\cite{baghdadi2020tiramisu}, and SparseRT~\cite{wang2020sparsert} involve searching for an appropriate kernel configuration for a specific sparsity pattern. These techniques can achieve high performance for a given sparsity pattern but fail to handle dynamic sparsity patterns, which are only known at runtime. The key contribution of \sysname{} is to support dynamic sparsity patterns on the fly through its low-overhead online sparsity detection and sparse-dense data transformation (\ie SRead and SWrite). Moreover, for most static sparsity patterns, \sysname{} can achieve performance similar to the ahead-of-time compilers, even though \sysname{} can only detect the sparsity patterns at runtime, greatly saving the compiling overhead.}

%their high overhead makes them unsuitable for optimizing dynamic sparsity patterns.

\noindent\revise{\textbf{Optimizations for General Sparsity Patterns.} 
Although general sparse libraries do not require ahead-of-time compilation, \eg OpenAI's block sparse kernel~\cite{triton}, cuSPARSE ~\cite{cusparse}, cuSPARSELT~\cite{cusparselt}, HipSparse~\cite{hipsparse}, Sputnik~\cite{sputnik}, nmSparse~\cite{nmsparse}, these libraries only support or work effectively on specific data granularity and computation granularity, making them difficult to support more complex sparsity patterns. For example, OpenAI's block sparse kernel~\cite{triton} only supports sparse data blocks of $32\times32$, leading to wasted computation when the sparsity pattern is more fine-grained (\eg $1\times32$). \sysname{} solves the problem by micro-tile. \sysname{} supports the online construction of larger tiles by sparsely reading/writing multiple micro-tiles, which achieves both high GPU efficiency and low sparsity coverage cost. Although ASpT~\cite{hong2019adaptive} proposes an adaptive tiling mechanism for fine-grained sparse matrix multiplication, it introduces significant offline data rearrangement overheads. But \sysname{} can construct the tile with negligible overheads as long as a micro-tile can saturate the GPU memory transaction.}

\noindent\revise{\textbf{Hardware Optimizations for Sparsity Computation.}
In addition to the above software solutions, many hardware optimizations have been proposed for the sparse computation of deep learning models~\cite{williams2007optimization, bell2008efficient, bulucc2009parallel, sgk_sc2020, sparsetensorcore}. %Flexagon~\cite{munoz2023flexagon} has proposed a sparse accelerator that dynamically adjusts the data flow based on the input. DRT~\cite{odemuyiwa2023accelerating} proposes a sparse MatMul accelerator that can dynamically change the tile size based on the input tensor. 
These hardware optimization techniques often target specific sparsity patterns. For example, NVIDIA's Sparse Tensor Core~\cite{sparsetensorcore} in A100 GPUs only supports a strict 2-in-4 pattern, \ie every 
$1\times4$ tile should have exactly 2 zeros, which greatly limits the applicability to more diverse sparsity patterns. \sysname{} not only does not assume the sparsity pattern but also has the ability to augment the limited sparsity patterns of existing hardware solutions. For instance, when a matrix has mixed multiple 1x4 tiles with two sparsity patterns: two zeros and all zeros, \sysname{} can construct micro-tiles to only feed the two-zero cases to the Sparse Tensor Core, avoiding the unnecessary computation of all-zero tiles. Such a GPU kernel can be implemented by combining \sysname{}'s SRead/SWrite with the “mma.sp” PTX instruction of Sparse Tensor Core, which we leave as future work.}

\noindent\revise{\textbf{Optimization for LLMs.} Among recent advances to optimize the inference of large language models, vLLM~\cite{vllm} is the most relevant to \sysname{}. vLLM proposed Paged Attention to ``sparsely'' load/save tokens from/to different physical addresses of GPU memory, breaking the continuous storage limitation of tokens. It saves excessive padding of varying sequence lengths and redundant copies during beam search. Paged Attention can be treated as a domain-specific solution for the special dynamic sparsity pattern in generative language models. \sysname{} can be used to implement vLLM with a customized PIT transformation policy. By design, \sysname{} is a general solution for dynamic sparsity that facilitates the  support of varying sequence lengths and more challenging sparsity patterns (e.g., MoE, sparse attention, sparse training).}

%% file: Discussion.tex
% \section{Discussion}
% \revise{In addition to efficient dense computation tiles, \sysname{} transformation can also be applied to other sparse kernel instructions, such as the Sparse TensorCore instructions in A100~\cite{sparsetensorcore}. \sysname{} can extend Sparse Tensor Core’s strict 2-in-4 pattern requirement to more general patterns. For instance, when a matrix has mixed multiple 1x4 tiles with two sparsity patterns: two zeros and all zeros, Sparse Tensor Core would unnecessarily compute the all-zero tiles. While \sysname{} can skip them, allowing more efficient utilization of Sparse Tensor Core for more flexible dynamic patterns. PIT can support Sparse Tensor Core by implementing a kernel using the “mma.sp” PTX instruction. We consider it as future work.}

% \revise{\sysname{} will select different micro-tile shapes and dense computation tiles accordingly on the fly, as shown in \autoref{alg:selection}. When the sparsity ratio of input is relatively low, the speed benefits brought about by sparsity might be offset by convert overhead. In such cases,  \sysname{} falls back to dense computation. \sysname{} can choose different runtime policies according to the predicted sparsity ratio of the current input. This prediction can be derived from the average sparsity ratios of inputs over a recent time. Alternatively, \sysname{} can also accommodate user-defined predictors for estimating input sparsity ratios. For instance, the sparsity ratio caused by varying sentence lengths can be straightforwardly computed from the lengths of all sentences within the current batch. }